\documentclass{article} % For LaTeX2e
\usepackage{nips14submit_e,times}
\usepackage{hyperref}
\hypersetup{
    bookmarks=true,         % show bookmarks bar?
    unicode=false,          % non-Latin characters in Acrobat's bookmarks
    pdftoolbar=true,        % show Acrobat's toolbar?
    pdfmenubar=true,        % show Acrobat's menu?
    pdffitwindow=false,     % window fit to page when opened
    pdfstartview={FitH},    % fits the width of the page to the window
    pdftitle={My title},    % title
    pdfauthor={Author},     % author
    pdfsubject={Subject},   % subject of the document
    pdfcreator={Creator},   % creator of the document
    pdfproducer={Producer}, % producer of the document
    pdfkeywords={keyword1} {key2} {key3}, % list of keywords
    pdfnewwindow=true,      % links in new window
    colorlinks=true,       % false: boxed links; true: colored links
    linkcolor=red,          % color of internal links (change box color with linkbordercolor)
    citecolor=blue,        % color of links to bibliography
    filecolor=magenta,      % color of file links
    urlcolor=cyan           % color of external links
}
\usepackage{url}
\usepackage{mathtools}

\usepackage[disable]{todonotes}
\usepackage[round,comma]{natbib}
\usepackage{url}
\usepackage{amsmath}
\usepackage{amsfonts}
\usepackage{amssymb}
\usepackage{amsthm}
\usepackage{graphicx}
\usepackage{algorithm,algorithmic}
\usepackage{enumerate}
\usepackage[capitalize]{cleveref}

\title{On Minimax Optimal Offline Policy Evaluation}

\author{
Lihong Li \\
Microsft Research \\
\texttt{lihongli@microsoft.com} \\
\And
Remi Munos \\
INRIA \\
\texttt{remi.munos@inria.fr} \\
\AND
Csaba Szepesv\'ari \\
University of Alberta \\
\texttt{szepesva@cs.ualberta.ca} \\
}

\newcommand{\todol}[2][]{\todo[size=\tiny,color=red!20!white,#1]{Lihong: #2}}
\newcommand{\todoc}[2][]{\todo[size=\tiny,color=green!20!white,#1]{Csaba: #2}}

\newcommand{\beq}{\begin{equation}}
\newcommand{\eeq}{\end{equation}}

\newcommand{\beqa}{\begin{eqnarray}}
\newcommand{\eeqa}{\end{eqnarray}}

\newcommand{\beqan}{\begin{eqnarray*}}
\newcommand{\eeqan}{\end{eqnarray*}}

\newcommand{\vlr}{\hat{v}_{\mathrm{LR}}}
\newcommand{\vreg}{\hat{v}_{\mathrm{Reg}}}
\newcommand{\valg}{\hat{v}_\Aalg}
\newcommand{\vtrue}{v^\pi_\Phi}

\newcommand{\vtruectx}{v^{\pi,\mu}_\Phi}

\newcommand{\rmax}{R_\mathrm{max}}
\newcommand{\defeq}{:=}

\renewcommand{\P}{\mathbb{P}}

\newcommand{\E}{\mathbb{E}}

\newcommand{\V}{\mathbb{V}}

\newcommand{\1}{\mathbb{I}}

\newcommand{\one}[1]{\mathbb{I}\{#1\}}
\newcommand{\EE}[1]{\E\left[#1\right]}
\newcommand{\Prob}[1]{\P\left(#1\right)}

\let\R\undefined %sometimes it is defined as something I don't know
\newcommand{\R}{\mathbb{R}}
\newcommand{\Real}{\mathbb{R}}

\newcommand{\MSE}[1]{\mathrm{MSE}\left(#1\right)}
\newcommand{\MSEi}[2]{\mathrm{MSE}_{#1}\left(#2\right)}

\newtheorem{predefinition}{Definition}

\newtheorem{theorem}{Theorem}

\newtheorem{preproposition}{Proposition}
\newenvironment{proposition}[1]
{
\begin{preproposition}
}{
\end{preproposition}
}

\newtheorem{lemma}{Lemma}
\newtheorem{corollary}{Corollary}

\renewcommand{\hat}{\widehat}
\renewcommand{\phi}{\varphi}
\renewcommand{\epsilon}{\varepsilon}

\newcommand{\A}{\mathcal{A}}

\newcommand{\F}{\mathcal{F}}

\newcommand{\X}{\mathcal{X}}

\newcommand{\T}{\mathcal{T}}

\newcommand{\norm}[1]{\left\|#1\right\|}

\newcommand{\Aalg}{\mathbf{A}}
\DeclareMathOperator{\diag}{diag}

\nipsfinalcopy % Uncomment for camera-ready version

\begin{document}

\maketitle

\begin{abstract}
This paper studies the off-policy evaluation problem, where one aims to estimate the value of a target policy based on a sample of observations collected by another policy.  We first consider the multi-armed bandit case, establish a minimax risk lower bound, and analyze the risk of two standard estimators.  It is shown, and verified in simulation, that one is minimax optimal up to a constant, while another can be arbitrarily worse, despite its empirical success and popularity.  The results are applied to related problems in contextual bandits and fixed-horizon Markov decision processes, and are also related to semi-supervised learning.
\end{abstract}

%!TEX root =  paper.tex
\section{Introduction} \label{sec:intro}

In reinforcement learning, one of the most fundamental problems is \emph{policy evaluation} --- estimate the average reward obtained by running a given policy to select actions in an unknown system.  A straightforward solution is to simply run the policy and measure the rewards it collects.  In many applications, however, running a new policy in the actual system can be expensive or even impossible.  For example, flying a helicopter with a new policy can be risky as it may lead to crashes; deploying a new ad display policy on a website may be catastrophic to user experience; testing a new treatment on patients may simply be impossible for legal and ethical reasons; \textit{etc.}

These difficulties make it critical to do \emph{off-policy} policy evaluation~\citep{Precup00Eligibility,Sutton10Convergent}, which is sometimes referred to as \emph{offline evaluation} in the bandit literature~\citep{Li11Unbiased} or counterfactual reasoning~\citep{Bottou13Counterfactual}.  Here, we still aim to estimate the average reward of a target policy, but instead of being able to run the policy online, we only have access to a sample of observations made about the unknown system, which may be collected in the past using a \emph{different} policy.  Off-policy evaluation has been found useful in a number of important applications~\citep{Langford08Exploration,Li11Unbiased,Bottou13Counterfactual} and
can also be looked as a key building block for policy \emph{optimization} which, as in supervised learning, can often be reduced to evaluation, as long as the complexity of the policy class is well-controlled~\citep{NgJo:00}. 
For example, it has played an important role in many optimization algorithms for Markov decision processes~(e.g., \citealt{HeiMeiIg09}) %\citep{Lagoudakis03Least} \todol{More recent references for MDPs?  Greedy-GQ?} 
and bandit problems~\citep{Auer02Nonstochastic,Langford08Epoch,Strehl11Learning}.
\todol{Others?  Examples for partial-monitor games?}
In the context of supervised learning, in the covariate shift literature, 
the problem of estimating losses under 
changing distributions is crucial for model selection \citep{sugimu05,Yu12Analysis} and also appears in active learning
\citep{Das11}.
In the statistical literature, on the other hand, the problem appears in the context of randomized experiments.
Here, the focus is on the two-action (binary) case where the goal is to estimate the difference between the expected rewards of the two actions \citep{Hirano03Efficient}, which is slightly (but not essentially) different than our setting.
\todoc{Mention in conclusion generalization}

The topic of the present paper is off-policy evaluation in finite settings, under a mean squared error criterion (MSE).
As opposed to the statistics literature \citep{Hirano03Efficient}, we are interested in results for \emph{finite sample sizes}.
In particular, we are interested in limits of performance (minimax MSE) given fixed policies, but unknown stochastic rewards with bounded mean reward, as well as the performance of estimation procedures compared to the minimax MSE.
We argue that the finite setting is not a key limitation when focusing on the scaling behavior of the MSE of algorithms. Moreover, we are not aware of prior work that would have studied the above problem (i.e., relating the MSE of algorithms to the best possible MSE).
Our main results are as follows:
We start with a lower bound on the minimax MSE, to set a target for the estimation procedures.
Next, we derive the exact MSE of the likelihood ratio (or importance-weighted) estimator (LR), which is shown to have an extra (uncontrollable) factor as compared to the minimax MSE lower bound. Next, we consider the estimator which estimates the mean rewards by sample means, which we call the regression estimator (REG). The motivation of studying this estimator is both its simplicity and also because it is known that a related estimator is asymptotically efficient \citep{Hirano03Efficient}. The main question is whether the asymptotic efficiency transfers into finite-time efficiency. 
Our answer to this is mixed: We show that the MSE of REG is within a constant factor of the minimax MSE lower bound, however, the ``constant'' depends on the number of actions ($K$), or a lower bound on the variance. 
We also show that the dependence of the MSE of REG on the number actions is unavoidable.
In any case, for ``small'' action sets or high noise setting, the REG estimator can be thought of as a minimax near-optimal estimator. 
We also show that for small sample sizes (up to $\sqrt{K}$) all estimators must suffer a constant MSE. 
Numerical experiments illustrate the tightness of the analysis.
Implications for more complicated settings, such as policy evaluation in contextual bandits and Markov Decision Processes (MDPs).
The question of designing a nearly minimax estimator independently of any problem parameters remains open.
All the proofs ot given in the main text can be found in the supplementary material.
\section{Multi-armed Bandit}
\label{sec:mab}

Let $\A=\{1,2,\ldots,K\}$ be a finite set of $K$ actions.  
%The process generating the 
Data $D^n=\{(A_i,R_i)\}_{1\le i\le n}$ is generated by the following process:
%described by the following:%
\footnote{The data $D^n$ is actually a list, not a set. We keep the notation $\{(A_i,R_i)\}_{1\le i\le n}$ for historical reasons.}
$(A_i,R_i)$ are independent copies of $(A,R)$, where $\Prob{A=a} = \pi_D(a)$ and $R \sim \Phi(\cdot|A)$ for some unknown family of distributions $\{\Phi(\cdot|a)\}_{a\in \A}$ and \emph{known} policy $\pi_D$.
%for each $i$, $A_i \sim \pi_D(\cdot)$ and $R_i \sim \Phi(\cdot|A_i)$, for some \emph{unknown} distribution $\Phi$ and \emph{known} policy $\pi_D$. 
 We are also given a \emph{known} target policy $\pi$ and want to estimate its value, 
 $\vtrue\defeq \E_{A\sim\pi,R\sim\Phi(\cdot|A)}[R]$ based on the knowledge of 
  $D^n$, $\pi_D$ and $\pi$, where the quality of an estimate $\hat{v}$ constructed based on $D^n$ (and $\pi,\pi_D$) is measured by its mean-squared error, $\MSE{\hat{v}}\defeq \EE{ (\hat{v}-\vtrue)^2 }$.
%  , the problem is to estimate $\vtrue$ with small mean squared error; that is, we aim to minimize the \emph{risk} defined as $\E_{D^n}[(\hat{v}-\vtrue)^2]$, where $\hat{v}$ is the estimate.

Define $r_\Phi(a)\defeq\E[R|A=a]$ and $\sigma^2_\Phi(a)\defeq\V(R|A=a)$, where $\V(\cdot)$
%= \EE{ (X-\EE{X})^2 }$ stands for the variance of a random variable $X$. 
stands for the variance.
Further, let $\pi_D^*:=\min_a\pi_D(a)$.  For convenience, we will identify any function $f: \A \to \Real$ with the $K$-dimensional vector whose $k$th component is $f(k)$. Thus, $r_\Phi$, $\sigma^2_\Phi$, etc. will also be looked at as vectors.
 Note that we do not assume that the rewards are bounded from either direction. \todoc{Really?}  \todol{Where to mention $\rmax$?  Shall we assume $r\ge0$ here?}
 
A few quantities are introduced to facilitate discussions that follow:
\begin{eqnarray*}
%\pi_D^* &\defeq& \min_a \pi_D(a)\,, \\
V_1 &\defeq& \E\left[\V\left(\frac{\pi(A)}{\pi_D(A)}R|A\right)\right] = \sum_a \frac{\pi^2(a)}{\pi_D(a)}\sigma_\Phi^2(a)\,, \\
V_2 &\defeq& \V\left(\E\left[\frac{\pi(A)}{\pi_D(A)}R|A\right]\right) 
= \V\left( \frac{\pi(A)}{\pi_D(A)} r_\Phi(A) \right)
= \sum_a\frac{\pi^2(a)}{\pi_D(a)}r_\Phi(a)^2 - (v_\Phi^\pi)^2 \,.
\end{eqnarray*}
Note that $V_1$ and $V_2$ are functions of $\Phi,\pi_D$ and $\pi$, but this dependence is suppressed.
%It is also important to note that 
Also, $V_1$ and $V_2$ are independent in that there are no constants $c,C>0$ such that $cV_1\le  V_2 \le C V_1$ for any $\pi,\pi_D,\Phi$.
%\sigma^2_\phi$, $r_{\phi}$.  
Finally, let $p_{a,n}\defeq(1-\pi_D(a))^n$ be the probability of having \emph{no} sample of $a$ in $D^n$.
	% Subsections:
	%!TEX root =  paper.tex
\subsection{A Minimax Lower Bound}
\label{sec:mab-lb}

We start with establishing a minimax lower bound that characterizes the inherent hardness of the off-policy evaluation problem.
% when combined with the upper bound given later in this section.
An estimator $\Aalg$ can be considered as a function that maps $(\pi,\pi_D,D^n)$ to an estimate of $v^\pi_\Phi$, denoted $\valg(\pi,\pi_D,D^n)$.  
Fix $\sigma^2 \defeq (\sigma^2(a))_{a\in \A}$.
We consider the minimax optimal risk subject to $\sigma^2_\Phi(a)\le \sigma^2(a)$ %
\todoc{Sometimes we use $\V(\Phi)$, sometimes $\sigma_\Phi^2$. Unify the notation or tell the reader to anticipate this.}
and $0 \le r_\Phi(a) \le \rmax$ for all $a\in \A$:
\[
R_n^* (\pi,\pi_D,\rmax,\sigma^2):= \inf_\Aalg \sup_{\Phi:\sigma^2_\Phi \le \sigma^2,0 \le r_\Phi \le \rmax} \E \left[(\valg(\pi,\pi_D,D^n)-v^\pi_\Phi)^2\right]\,,
\]
where for vectors $x,y\in \Real^K$, $x\le y$  holds if and only if $x_i\le y_i$ for $1\le i \le K$.
For $B\subset \A$, we let $p_{B,n}$ denote the probability that none of the actions in the data $D^n$ falls into $B$: $p_{B,n} = \Prob{A_1,\ldots,A_n\not\in B}$. Note that this definition generalizes $p_{a,n}$. We also let $\pi(B) = \sum_{a\in B} \pi(a)$.
%Using an information-theoretic argument, we  establish the following lower bound:
%\todoc{Add lower bound that depends on $p_{a,n}$.}

\begin{theorem}
\label{thm:lb}
For any $n>0$, $\pi_D$, $\pi$, $\rmax$ and $\sigma^2$, one has
\[
R_n^* (\pi,\pi_D,\rmax,\sigma^2)\ge \frac{1}{4} \max\left(\rmax^2 \max_{B\subset \A} \pi^2(B) p_{B,n}, \frac{V_1}{n}\right) \,.
\]
%\todoc{Is the constant universal?}
Furthermore,
\begin{align}
\label{eq:aslb}
\liminf_{n\to \infty} \frac{R_n^*(\pi,\pi_D,\rmax,\sigma^2)}{V_1/n} \ge 1.
\end{align}
%$
%R_n^* (\pi,\pi_D,\sigma^2)= \Omega\left(\max\left\{(\exp(-n\pi_D^*), \frac{V_1}{n}\right\}\right) .
%$
\end{theorem}

\begin{proof}
\newcommand{\hv}{\hat{v}}
To prove the first part of the lower bound, fix a subset $B\subset \A$ of actions and choose an environment $\Phi\in \mathcal{E}$, where
$\mathcal{E}$ is the set of environments $\Phi$ such that $\sigma^2_\Phi\le \sigma^2$ and $0 \le r_\Phi \le \rmax$.
Introduce the notation $\E_{\Phi}$ to denote expectation when the data is generated by environment $\Phi$. \todol{Why call ``environment'', not ``distribution''?}

%\footnote{The condition $\sigma^2_\Phi= 0$ simplifies one step of the proof, but is not necessary.} 
%Then $D^n  = \{(A_1,r_\Phi(A_1)),\ldots,(A_n,r_\Phi(A_n))\}$.
Let $D^n$ be the data generated based on $\pi_D$ and $\Phi$ and
let $\hv_{\Aalg}(D^n)$ denote the estimate produced by some algorithm $\Aalg$.
Define $S = \{A_1,\ldots,A_n\}$ to be the set of actions in the dataset that is seen by the algorithm.
Clearly, for any $\Phi,\Phi'$ such that they agree on the \emph{complement} of $B$ (but may differ on actions in $B$),\todol{Agreement between $\Phi$ and $\Phi'$, not between $r_\Phi$ and $r_{\Phi'}$?}
\begin{align}
\label{eq:keylb}
\E_\Phi[ \hv_{\Aalg}(D_n) | S \cap B=\emptyset ] = \E_{\Phi'}[ \hv_{\Aalg}(D_n) | S \cap B=\emptyset ]\,.
\end{align}
Now,
$ \MSEi{\Phi}{\hv_{\Aalg}} \defeq \E_{\Phi}[ (\hv_{\Aalg}(D^n)-\vtrue)^2 ] \ge 
\E_{\Phi}[ (\hv_{\Aalg}(D^n)-\vtrue)^2 | S \cap B=\emptyset ] \Prob{ S \cap B = \emptyset }$ 
and by adapting the argument that the MSE is lower bounded by the bias squared,
$ \E_{\Phi}[ (\hv_{\Aalg}(D^n)-\vtrue)^2 | S \cap B=\emptyset ] \ge 
	(\E_{\Phi}[ \hv_{\Aalg}(D^n) | S \cap B=\emptyset ]-\vtrue)^2$.
Hence, $ 
\MSEi{\Phi}{\hv_{\Aalg}}\ge 
\Prob{ S \cap B = \emptyset }
\sup_{\Phi\in \mathcal{E} } (\E_{\Phi}[ \hv_{\Aalg}(D^n) | S \cap B=\emptyset ]-\vtrue)^2 $.
We get an even smaller quantity if we further restrict the environments $\Phi$ %that we take the maximum over 
to environments $\mathcal{E}_0$ that also satisfy $r_{\Phi} = \sigma^2_{\Phi} = 0$ on $\A \setminus B$. 
Now, by~\eqref{eq:keylb}, for all these environments, $\E_{\Phi}[ \hv_{\Aalg}(D^n) | S \cap B=\emptyset ]$ takes on a common value, denote it by $v_{\Aalg}$. Hence, 
$\MSEi{\Phi}{\hv_{\Aalg}}  \ge \Prob{ S \cap B = \emptyset } \sup_{\Phi \in \mathcal{E}_0} (v_{\Aalg}-\vtrue)^2$.
Since $\vtrue = \sum_{a\in B} \pi(a) r_{\Phi}(a)$, 
$\sup_{\Phi \in \mathcal{E}_0} (v_{\Aalg}-\vtrue)^2 \ge \frac{\rmax^2}{4} \pi^2(B)$, where we use the shorthand $\pi(B) = \sum_{a\in B} \pi(a)$. Plugging this into the previous inequality we get
$\sup_{\Phi\in \mathcal{E}} \MSEi{\Phi}{\hv_{\Aalg}}  \ge \Prob{ S \cap B = \emptyset }  \frac{\rmax^2}{4} \pi^2(B)$. 
Since $\Aalg$ was arbitrary, we get $R_n^*(\pi,\pi_D,\rmax,\sigma^2) \ge \Prob{ S \cap B = \emptyset }  \frac{\rmax^2}{4} \pi^2(B)$.

%any $a\in\A$, and consider two reward distributions, $\Phi$ and $\Phi'$: $r_\Phi(a')=0$ for all $a'$; $r_{\Phi'}$ is identical to $r_{\Phi}$ except that $r_{\Phi'}(a)=\rmax$.  The policy value is thus either $0$ or $\pi(a)\rmax$.  If $a$ is not sampled in $D^n$, there is no information to tell $\Phi$ from $\Phi'$, and any estimator will suffer a squared error of at least $\pi^2(a)\rmax^2/4$.  But the probability of such an event is exactly $p_{a,n}$.  So the MSE is at least $p_{a,n}\pi^2(a)\rmax^2/4$.  The above holds for all $a$, showing that the first part of the lower bound indeed holds.

%, so that $r_{\Phi_i}(a)=i\1\{a=1\}$.  Furthermore, define $\pi_D(1)=\pi_D^*$ and $\pi(1)=1$.  Then, the true policy value is either $0$ (for $\Phi_0$) or $1$ (for $\Phi_1$), and $\P(n(1)=0)=(1-\pi_D^*)^n\approx\exp(-n\pi_D^*)$.  Conditioned on the event $\{n(1)=0\}$, the minimax optimal estimate is to predict $1/2$ with an MSE of $1/4$.  It follows that the (unconditional) MSE is at least $\exp(-n\pi_D^*)/4$.  \todol{Rewrite the proof using the same info-theoretic argument to include $\sigma$ in the bound.} \todol{Minor tweaks needed since $\exp(-n\pi_D^*)$ is an upper, \emph{not} lower, bound of $(1-\pi_D^*)^n$.}

For the second part, consider a class of normal distributions with fixed reward variances $\sigma^2$ but different reward expectations: $\F_p = \{\Phi_0, \ldots, \Phi_{p-1}\}$, where $r_{\Phi_i}=2i\sqrt{\epsilon}\Delta\in\R^K$, 
%for some fixed $r\in \R$, \todoc{Get rid of $r$? Replace with $0$??}
for some to-be-specified vector $\Delta\in\R_+^K$ that satisfies $\sum_a\pi(a)\Delta(a)=1$.  The data-generating distribution $\Phi$ is in $\F_p$, but is unknown otherwise.

It is easy to see that the policy value between any two distributions in $\F_p$ differ by at least $2\sqrt{\epsilon}$.  Indeed, for any $\Phi_i,\Phi_j\in\F_p$, $|v^\pi_{\Phi_i}-v^\pi_{\Phi_j}| = 2\sqrt{\epsilon}|i-j|\sum_a\pi(a)\Delta(a) = 2\sqrt{\epsilon}|i-j| \ge 2\sqrt{\epsilon}$.  It follows that, in order to achieve a squared error less than $\epsilon$, one needs to identify the underlying data-generating $\Phi$ from $\F_p$, based on the observed sample $D^n$.  The problem now reduces to finding a minimax lower bound for hypothesis testing in the given finite set $\F_p$.

We resort to the information-theoretic machinery based on Fano's inequality (see, e.g., \cite{Raginsky11Information}).  Define an oracle which, when queried, outputs $Y=(A,R)$ with $A\sim\pi_D(\cdot)$ and $R\sim\Phi(\cdot|A)$.
Let the distribution of $Y$ when $\Phi$ is used be denoted by $\P_{Y|\Phi}$. Let $\F_p$ collect $p$ distributions such that $\Phi(\cdot|a)$ is normal.
 Consider $\Phi,\Phi'\in\F_p$.  Then, 
\[
D(\P_{Y|\Phi}\|\P_{Y|\Phi'})=\sum_a\pi_D(a)D( \Phi(\cdot|a)\| \Phi'(\cdot|a) ) = 2\epsilon(i-j)^2\sum_a\frac{\pi_D(a)\Delta(a)^2}{\sigma(a)^2} .
\]
The divergence measures how much information is carried in one sample from the oracle to tell $\Phi$ from $\Phi'$.  To obtain the tightest lower bound, we should minimize the divergence.  Subject to the constraint $\sum_a\pi(a)\Delta(a)=1$, the divergence is minimized by setting $\Delta(a)\propto\frac{\pi(a)}{\pi_D(a)}\sigma^2(a)$, and is $2\epsilon(i-j)^2/V_1$.  Now setting $p=6$, and applying Lemma~1, Theorem~1 and the ``Information Radius bound'' from \cite{Raginsky11Information}, we have $n\ge\frac{V_1}{4\epsilon}$.
%$n=\Omega(V_1/\epsilon)$.  
\todoc{Can we add these to the appendix?}
Reorganizing terms and combining with the first term  %(observing $\max\{x,y\}\ge(x+y)/2$) 
complete the proof of the first statement.

For the second part, note that it suffices to consider asymptotically unbiased estimators (cf. the generalized Cramer-Rao lower bound, Theorem~7.3 of \citealt{Ibramigov81StatEstBook}).
For any such estimator, the Cramer-Rao lower bound gives the result with
%with the following choices:
the parametric family chosen to be $p(a,y;\theta) = \pi_D(a) \phi( y; r(a), \sigma^2(a))$, where 
$\theta = (r(a))_{a\in A}$ is the unknown parameter to be estimated, and $\phi(\cdot;\mu,\sigma^2)$ is the density of the normal distribution with mean $\mu$ and variance $\sigma^2$ and
the quantity to be estimated is $\psi(\theta) = \sum_a \pi(a) r(a)$. For details, see \cref{sec:cr}.
\if0
Now, by the mentioned theorem, for any $\theta$ and any
unbiased estimator $\hat{v}$ based on the data $D^n$ generated from $f$,
$\MSE{\hat{v}} \ge \frac1n \psi'(\theta)^\top F^{-1}(\theta) \psi'(\theta)$, where $F(\theta)$ is the Fisher information matrix underlying $p(\cdot;\theta)$.
A direct calculation shows that $F(\theta) = \diag( \ldots, \pi_D(a)/\sigma^2(a), \ldots)$ and $\psi'(\theta) = \pi$. Hence, the lower bound becomes $\psi'(\theta)^\top (nF)^{-1}(\theta) \psi(\theta) = V_1/n$, finishing the proof.
\fi
\end{proof}
The next corollary says that the minimax risk is constant when the number of samples is $O(\sqrt{K})$:
\begin{corollary}
For $K\ge 2$, $n\le \sqrt{K}$, $\sup_{\pi} R_n^*(\pi,\pi_D,\rmax,\sigma^2) = \Omega(\rmax^2)$.
\end{corollary}
\begin{proof}
Choose $B\subset \A$ to minimize $\pi_D(B)$ subject to the constraint $|B| = \lfloor\sqrt{K}\rfloor$.
Note that $\Prob{A_1,\ldots,A_n \not\in B} = (1-\pi_D(B))^n \ge (1-|B|/K)^n \ge (1-1/\sqrt{K})^{\sqrt{K}} \ge (1-1/\sqrt{2})^{\sqrt{2}}$.
Choosing $\pi$ such that $\pi(B) = 1$ gives the result.
%By~\cref{lem:emptybincount2}, for any $\pi_D$ and $n\le K/2$, 
%with probability at least $1-e^{-1}$, the number of actions that did not receive any observations is at least $K/4$.
\end{proof}
We conjecture that the result can be strengthened by increasing the upper limit on $n$. \todoc{One of the conjectures; mention in conclusion}

	%!TEX root =  paper.tex
\subsection{Likelihood Ratio Estimator}
\label{sec:mab-lr}

One of the most popular estimators is known as 
 the propensity score estimator in the statistical literature \citep{roru83,roru85},
or the importance weighting estimator \citep{Bottou13Counterfactual}.
We call it the likelihood ratio estimator, as it estimates the unknown value using likelihood ratios, or importance weights:
% ~\citep{Gretton08Covariate,Strehl11Learning,Yu12Analysis}.  
%This estimator employs an importance sampling correction to counter the distributional mismatch between $\pi_D$ and $\pi$:
\[
\vlr(\pi,\pi_D,D^n) \defeq \frac{1}{n}\sum_{i=1}^n \frac{\pi(A_i)}{\pi_D(A_i)} R_i .
\]
Its distinguishing feature  is that it is \emph{unbiased}: $\E[\vlr(\pi,\pi_D,D^n)]=v^\pi_\Phi$, implying that the MSE is purely contributed by the variance of the estimator.  The main result in this subsection shows that this estimator does not achieve the minimax lower bound up to \emph{any} constant (by making $V_2 \gg V_1$).  The proof (given in the appendix) is based on a direct calculation using the law of total variance.

\begin{proposition}{}
\label{prop:lrmse}
It holds that $\MSE{\vlr(\pi,\pi_D,D^n)} = (V_1+V_2)/n$ .
\end{proposition}
We see that as compared to the lower bound on the minimax MSE, an extra $V_2/n$ factor appears. In the next section, we will see that this factor is superfluous, showing that the MSE of LR can be ``unreasonably large''.

%It is easy to construct cases where $V_2\gg V_1$, implying a non-constant, multiplicative gap between the MSE of this estimator and the lower bound in the previous subsection.  As will be shown soon, the lower bound can be achieved (up to a constant factor), so the likelihood ratio estimator is \emph{not} minimax optimal.
	%!TEX root =  paper.tex
\subsection{Regression Estimator}
\label{sec:reg}

For convenience, define $n(a) \defeq \sum_{i=1}^n \1(A_i=a)$ to be the number of samples for action $a$ in $D^n$, and $R(a) \defeq \sum_{i=1}^n \1(A_i=a) R_i$ the total rewards of $a$.
The regression estimator (REG) is given by
\[
\vreg(\pi,D^n) \defeq \sum_{a} \pi(a) \hat{r}(a), 
\quad \mbox{where} \quad
\hat{r}(a) \defeq \begin{cases}
0, & \text{if } n(a) = 0;\\
\frac{R(a)}{n(a)}, & \text{otherwise}\,.
\end{cases}
\]
For brevity, we will also write $\hat{r}(a) = \one{n(a)>0} \frac{R(a)}{n(a)}$, where we take $\frac{0}{0}$ to be zero.
The name of the estimator comes from the fact that it estimates the reward function, and the problem of estimating the reward function can be thought of as a regression problem.

Interestingly, as can be verified by direct calculation, the REG estimator can also be written as
\begin{align}
\label{eq:reglr}
\vreg(\pi,D^n) = \frac1n \sum_{i=1}^n \frac{\pi(A_i)}{\hat{\pi}_D(A_i)} R_i\,,
\end{align}
where $\hat\pi_D(a) =\frac{n(a)}{n}$ is the empirical estimate of $\pi_D(a)$.  Hence, the main difference between LR and REG is that the former uses $\pi_D$ to reweight the data, while the latter uses the \emph{empirical estimates $\hat{\pi}_D$}.
%latter uses $\pi_D$, while the former uses the empirical estimates when reweighting the data. As we will see, this small change results in major differences.
It may appear that LR is superior since it uses the ``right'' quantity.  Surprisingly, REG turns out to be much more robust than LR, as will be shown shortly; further discussion is made in Section~\ref{sec:ssl}.

For the next statement, the counterpart of \cref{prop:lrmse}, the following quantities will be useful:
%will prove to be useful: $p_{a,n} \defeq \P(n(a)=0)$,
%$v^\pi_{|\Phi|} \defeq \sum_a \pi(a) |r_\Phi(a)|$, 
\begin{align*}
V_{0,n}& \defeq \left(\sum_a\pi(a)r_\Phi(a)p_{a,n}\right)^2+ \sum_a \pi^2(a) r_\Phi^2(a) \,  p_{a,n}(1-p_{a,n})\, \quad\text{ and }\\
V_{3,n} &\defeq \sum_a\E\left[\frac{\1\{n(a)>0\}}{\hat\pi_D(a)}-\frac{1}{\pi_D(a)}\right]\pi(a)^2\sigma^2(a)\,.
\end{align*}
%where  again, in $\frac{\1\{n(a)>0\}}{\hat\pi_D(a)}$ we define $\frac{0}{0}$ to be zero.

\begin{proposition}{} \label{prop:regmse}
Fix $\pi,\pi_D$. Assume that $r_\Phi$ is nonnegative valued.
\todol{How about making it an assumption in the general setting?}
Then it holds that $\MSE{\vreg(\pi,D^n)} \le V_{0,n} + (V_1+V_{3,n})/n$.
Further, for any $\Phi$ such that the rewards have normal distributions, 
defining $b_n = \sum_a \pi(a) r_\Phi(a) p_{a,n}$ to be the bias of $\vreg$,
$
\MSE{\vreg} \ge \frac{V_1}{n} + 4 b_n^2 \left(1+\frac{V_1}{n}\right) + \frac{2}{n} \sum_a \frac{\pi^2(a)}{\pi_D(a)} \sigma^2_\Phi(a) p_{a,n}
$.
\end{proposition}

\begin{proof}[Proof sketch]
For the upper bound  use that 
the MSE equals the sum of squared bias and the variance.  It can be verified that REG is slightly biased: $\E[\vreg]-v_\Phi^\pi = \sum_a\pi(a)r_\Phi(a)p_{a,n}$.
For the variance term, we use the law of total variance to yield: $\V(\vreg) = \E[\V(\vreg|n(1),\ldots,n(K))] + \V(\E[\vreg|n(1),\ldots,n(K)])$, where the first term is $\sum_a \pi^2(a) \sigma^2(a) \E[\1\{n(a)>0\}/n(a)]$, and the second term is upper bounded (Lemma~\ref{lem:reg-var-ub}) by $\sum_a \pi^2(a) r_\Phi^2(a) \,  p_{a,n}(1-p_{a,n})$.  The proof is then completed by adding squared bias to variance, and using definitions of $V_{0,n}$, $V_1$, and $V_3$.
The lower bound follows from the (generalized) Cramer-Rao inequality. 
\end{proof}

The main result of this section is the following theorem that characterizes the MSE of REG in terms of the minimax optimal MSE.
\begin{theorem}[Minimax Optimality of the Regression Estimator]
\label{thm:regrminimax}
The following hold:
\begin{enumerate}[(i)]
\item \label{thm:regrminimax:part1}
For any $\pi,\pi_D$,  $\sigma^2 = (\sigma^2(a))_{a\in A}$,
$\Phi$ such that $\min_a r_{\Phi}(a)\ge 0$, $\max_a r_\Phi(a)\le\rmax$, and $\sigma_\Phi^2\le \sigma^2$, 
it holds for any $n>0$ that
\begin{align}
\label{eq:mseregbound}
\MSE{\vreg(\pi,D_n)} \le K \left\{ \min(4K, \max_a \tfrac{r_\Phi^2(a)}{\sigma_\Phi^2(a)}) +5\right\}\, R_n^*(\pi,\pi_D,\rmax,\sigma^2)\,,
\end{align}
where $D_n = \{(A_i,R_i)\}_{i=1,\ldots,n}$ is an i.i.d. sample from $(\pi_D,\Phi)$.
\item \label{thm:regrminimax:part2}
A suboptimality factor of $\Omega(K)$ in the above result is unavoidable: 
For $K>2$, there exists $(\pi,\pi_D)$ such that for any $n\ge 1$,
\[
\frac{\MSE{\vreg(\pi,D_n)}}{R_n^*(\pi,\pi_D,\rmax,0)}  \ge n e^{-2n/(K-1)}\,.
\]
Thus for $n=(K-1)/2$, this ratio is at least $\frac{K-1}{2e}$.
\item The estimator $\vreg$ is asymptotically minimax optimal:
\[
\limsup_{n\to\infty} \frac{\MSE{ \vreg(\pi,D_n) }}{ R_n^*(\pi,\pi_D,\rmax,\sigma^2) }\le 1\,.
\]
\end{enumerate}
\end{theorem}
We need the following lemma, which may be of interest on its own:
\begin{lemma}\label{lem:invmoment}
Let $X_1,\ldots,X_n$ be $n$ independent Bernoulli random variables with parameter $p>0$. 
Letting $S_n = \sum_{i=1}^n X_i$, $\hat{p} = S_n/n$, $Z = \frac{\one{S_n>0}}{\hat{p}} - \frac{1}{p}$, we have for any $n$ and $p$ that $\EE{ Z } \le 4/p$.  Further, when $np\ge 34$, we have
$
\label{eq:invmoment-2}
\EE{ Z } \le \frac{2}{p} \sqrt{\frac{2}{np}} \left( \sqrt{\frac32 \ln\left(\frac{np}{2}\right)}+1 \right)\,.
$
\end{lemma}

\begin{proof}[Proof of \cref{thm:regrminimax}]

First, we bound $V_{3,n}$ in terms of $V_1$.
From \cref{lem:invmoment}, 
$
\E\left[\frac{\1\{n(a)>0\}}{\hat\pi_D(a)}-\frac{1}{\pi_D(a)}\right] \le \frac{4}{\pi_D(a)}
$,
while if $n\pi_D^*\ge 34$,
$
\E\left[\frac{\1\{n(a)>0\}}{\hat\pi_D(a)}-\frac{1}{\pi_D(a)}\right] \le
\frac{2}{\pi_D(a)} 
\sqrt{\frac{2}{n\pi_D(a)}} \left( \sqrt{\frac32 \ln\left(\frac{n\pi_D(a)}{2}\right)}+1 \right)
$.
Plugging these into the definition of $V_{3,n}$, we have $V_{3,n} \le 4 V_1$ for all $n$.  Furthermore, when $n\pi_D^*\ge 34$, thanks to monotonicity of the function $t \mapsto \sqrt{\frac{2}{t}}\left(\sqrt{\frac{3}{2}\ln t}+1\right)$ for $t>1$, we have
\begin{align}
\label{eq:v3nasb}
V_{3,n} \le  2 V_1 \sqrt{\frac{2}{n\pi_D*}} \left( \sqrt{\frac32 \ln\left(\frac{n\pi_D^*}{2}\right)}+1 \right)\,.
\end{align}
Now, to bound 
$V_{0,n} =\left(\sum_a\pi(a)r_\Phi(a)p_{a,n}\right)^2+ \sum_a \pi^2(a) r_\Phi^2(a) \,  p_{a,n}(1-p_{a,n})$, remember that 
one lower bound for $R_n^*$ is $R_{\max}^2 \max_a \pi^2(a) p_{a,n}/4$,
where $R_{\max}$ is the range for $r_\Phi$. \todoc{The range has to be in the final result..}
Hence,
\begin{align}
V_{0,n} 
& =  K^2\left(\frac1K \sum_a\pi(a)r_\Phi(a)p_{a,n}\right)^2+
 \sum_a  \pi^2(a) r_\Phi^2(a) \,p_{a,n}(1-p_{a,n})\nonumber \\
& \le K \sum_a \pi^2(a)r_\Phi^2(a)p_{a,n}^2 + \sum_a\pi^2(a)r_\Phi^2(a)p_{a,n}(1-p_{a,n}) \nonumber \\
& \le K \sum_a \pi^2(a) r_\Phi^2(a) p_{a,n}
\le K^2 \max_a \pi^2(a) r_\Phi^2(a)  p_{a,n}\,.
 \label{eq:v0nasb}
\end{align}
Hence, using $R_n^*\ge V_1/n$, \todoc{constant?}
\begin{align}
\MSE{\vreg} 
	&\le V_{0,n} + \tfrac{V_1+V_3}{n} 
	\le 4K^2 \max_a \pi^2(a) r_\Phi^2(a)p_{a,n} + 5\tfrac{V_1}{n} \le  (4K^2+5) R_n^*\,.
\label{eq:msereg1}
\end{align}
On the other hand, assuming that $\min_a \sigma^2(a)>0$, we also have
\begin{align*}
V_{0,n}
& \le K \sum_a \pi^2(a) r_\Phi^2(a) p_{a,n}
 \le K \max_{b\in \A} \left(\tfrac{r_\Phi^2(b)}{\sigma^2(b)} \right)\, \sum_a p_{a,n} \pi^2(a) \sigma^2(a) \le K \max_{b\in \A} \left(\tfrac{r_\Phi^2(b)}{\sigma^2(b)} \right)\,\, \tfrac{V_1}{n}\,,
\end{align*}
where in the last inequality we used
that $p_{a,n}\le e^{-n\pi_D(a)}$ and $e^{-x}\le 1/x$, which is true for any $x>0$, and finally also the definition of $V_1$.
Similarly to the previous case, we get
\begin{align*}
\MSE{\vreg}
&\le  \left\{K\max_{b\in \A} \left(\tfrac{r_\Phi^2(b)}{\sigma^2(b)} \right) +5 \right\} \frac{V_1}{n} 
 \le \left\{ K \max_{b\in \A} \left(\tfrac{r_\Phi^2(b)}{\sigma^2(b)} +5 \right) \right\}R_n^*\,.
\end{align*}
Combining this with~\eqref{eq:msereg1} gives~\eqref{eq:mseregbound}.

For the second part of the result, choose $\pi(a) = \pi_D(a) = 1/K$, $r_\Phi(a) = 1$.
For $K\ge 2$, $p_{a,n} = (1-1/K)^n = e^{-n \log(1/(1-1/K))} = e^{-n \log(1+1/(K-1))} \ge e^{-n/(K-1)}$. 
Hence, we have 
$\MSE{\vreg} 
	\ge (\EE{\vreg - \vtrue})^2 
	= \left(\sum_a\pi(a)r_\Phi(a)p_{a,n}\right)^2 
	\ge e^{-2n/(K-1)}$.
Now, consider the LR estimator. Choosing $\sigma^2=0$, we have $V_1=0$
and so by \cref{prop:lrmse},
\begin{align*}
\sup_{
		\Phi: 0\le r_\Phi \le 1,
		\sigma^2_\Phi= 0} \MSE{\vlr} 
= 
\sup_{
		\Phi: 0\le r_\Phi \le 1,
		\sigma^2_\Phi= 0} V_2/n
 \le \frac1{n}\,.
\end{align*}
Hence,
$
\frac{\MSE{\vreg} }{R_n^*(\pi,\pi_D,1,0)} \ge \frac{e^{-2n/(K-1)}}{ \sup_{\Phi: 0\le r_\Phi \le 1, \sigma^2_\Phi= 0} \MSE{\vlr}  }
\ge n e^{-2n/(K-1)}
$.
%Hence, for $n\le K$, $\frac{\MSE{\vreg} }{R_n^*(\pi,\pi_D,1,0)} \ge K e^{-4}$.

Finally, the for the last part, 
fix any $\pi, \pi_D$, $\sigma^2$, $\Phi$ such that $\sigma^2_\Phi\le \sigma^2$.
Then, for $n$ large enough,
%\begin{align*}
$
\MSE{\vreg} 
\le V_{0,n} + \frac{V_1+V_3}{n}  
%& \le K(K+1) R_{\max}^2 e^{-n \pi_D^*} + \frac{V_1}{n} \,\,
%\left( 1+ 2 \sqrt{\frac{2}{n\pi_D*}} \left( \sqrt{\frac32 \ln\left(\frac{n\pi_D^*}{2}\right)}+1 \right) \right) \,.
 \le C e^{-n/C} + \frac{V_1}{n} \,\,
\left( 1+ C \sqrt{\frac{\ln n}{n}} \right) 
$,
%\end{align*}
where $C>0$ is a problem dependent constant, 
and the second inequality used \eqref{eq:v3nasb}  and \eqref{eq:v0nasb}.
Combining this with  \eqref{eq:aslb} of \cref{thm:lb} gives the desired result.
\if0
Hence,
\begin{align*}
\limsup_{n\to\infty} \frac{\MSE{\vreg}}{R_n^*(\pi,\pi_D,\sigma^2)} 
& \le \limsup_{n\to\infty} \frac{C e^{-n\pi_D^*} + (V_1/n) (1+C\sqrt{ \ln(n)/n}) }{R_n^*(\pi,\pi_D,\sigma^2)} \\
& \le \limsup_{n\to\infty} \frac{C e^{-n\pi_D^*} + (V_1/n) (1+C\sqrt{ \ln(n)/n}) }{V_1/n}  = 1\,,
\end{align*}
where the second inequality used \eqref{eq:aslb} of \cref{thm:lb}.
\fi
\end{proof}
%First, observe that, for some constant $\kappa$,
%\begin{eqnarray*}
%V_{3,n} &\le& \sum_a\pi(a)^2\sigma(a)^2\left(\frac{1}{\pi_D(a)-\kappa\sqrt{\pi_D(a)/n}} - \frac{1}{\pi_D(a)}\right) \\
%  &\approx& \sum_a \pi(a)^2\sigma(a)^2\frac{\kappa}{n^{1/2}\pi_D(a)^{3/2}} \le \frac{\kappa}{\sqrt{n\pi_D^*}} V_2 .
%\end{eqnarray*}
%Assuming $n \gg 1/\pi_D^*$, $V_{3,n}$ is dominated by $V_2$.
%

\if0
\textbf{Note: The following calculation of the ``crossing point'' $n^*$ ignores $V_{3,n}$ and needs to be updated.}
Therefore, the comparison between the two estimators essentially boils down to the comparison of $Ke^{-n\pi_D^*}$ and $V_2/n$.  Crossing happens when these two are equal, namely, when
\[
n = n^* := \frac{1}{\pi_D^*}\ln\frac{K}{V_2}\ln\left(K\ln\frac{K}{V_2}\right).
\]
For $n<n^*$, $\hat v_2$ has lower MSE; for $n>n^*$, $\hat v_1$ has lower MSE.
\fi

	%!TEX root =  paper.tex
\subsection{Simulation Results}

\begin{figure}
\begin{center}
\begin{tabular}{cc}
\includegraphics[width=0.45\columnwidth]{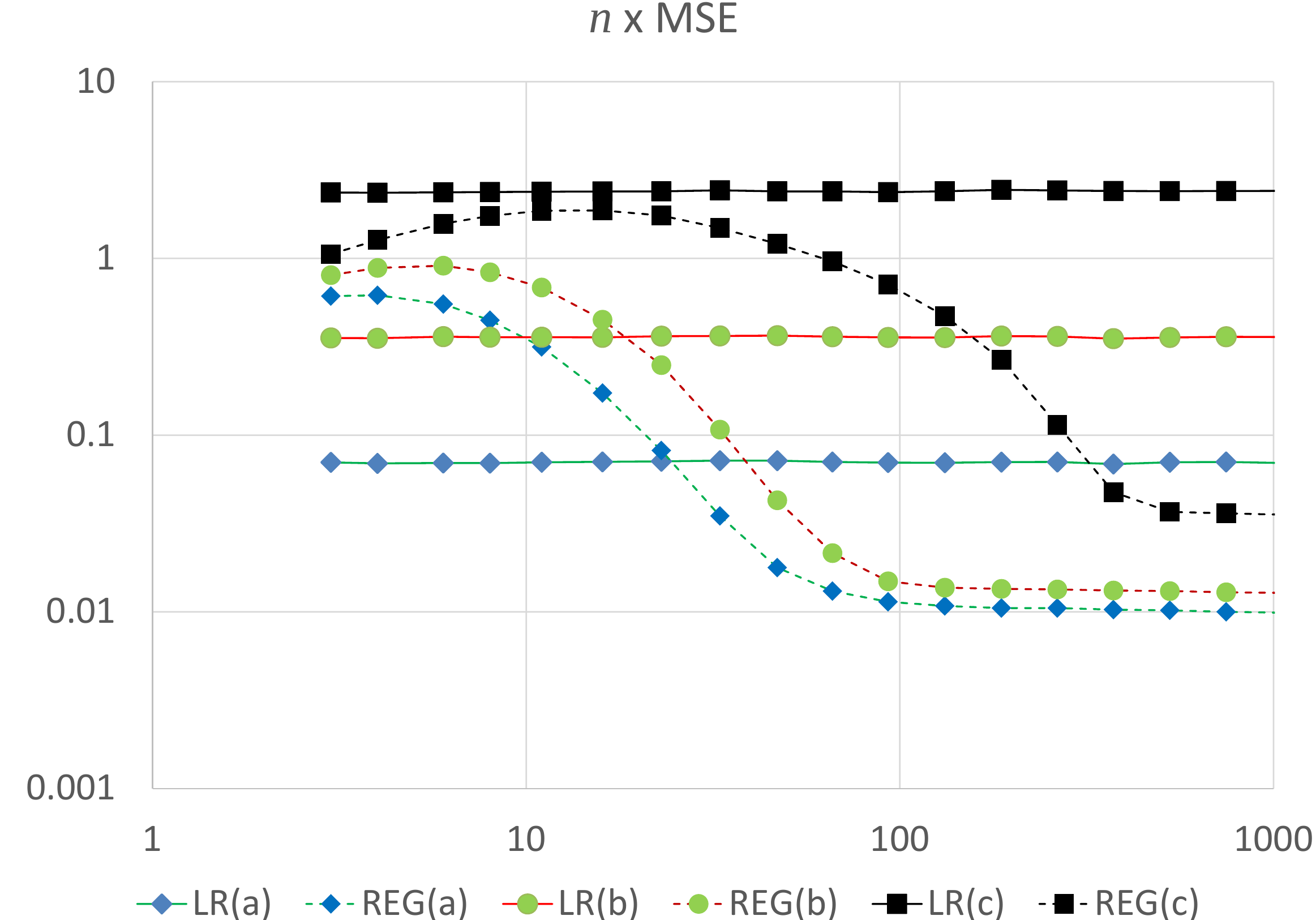}
\hspace{5mm}
\includegraphics[width=0.45\columnwidth]{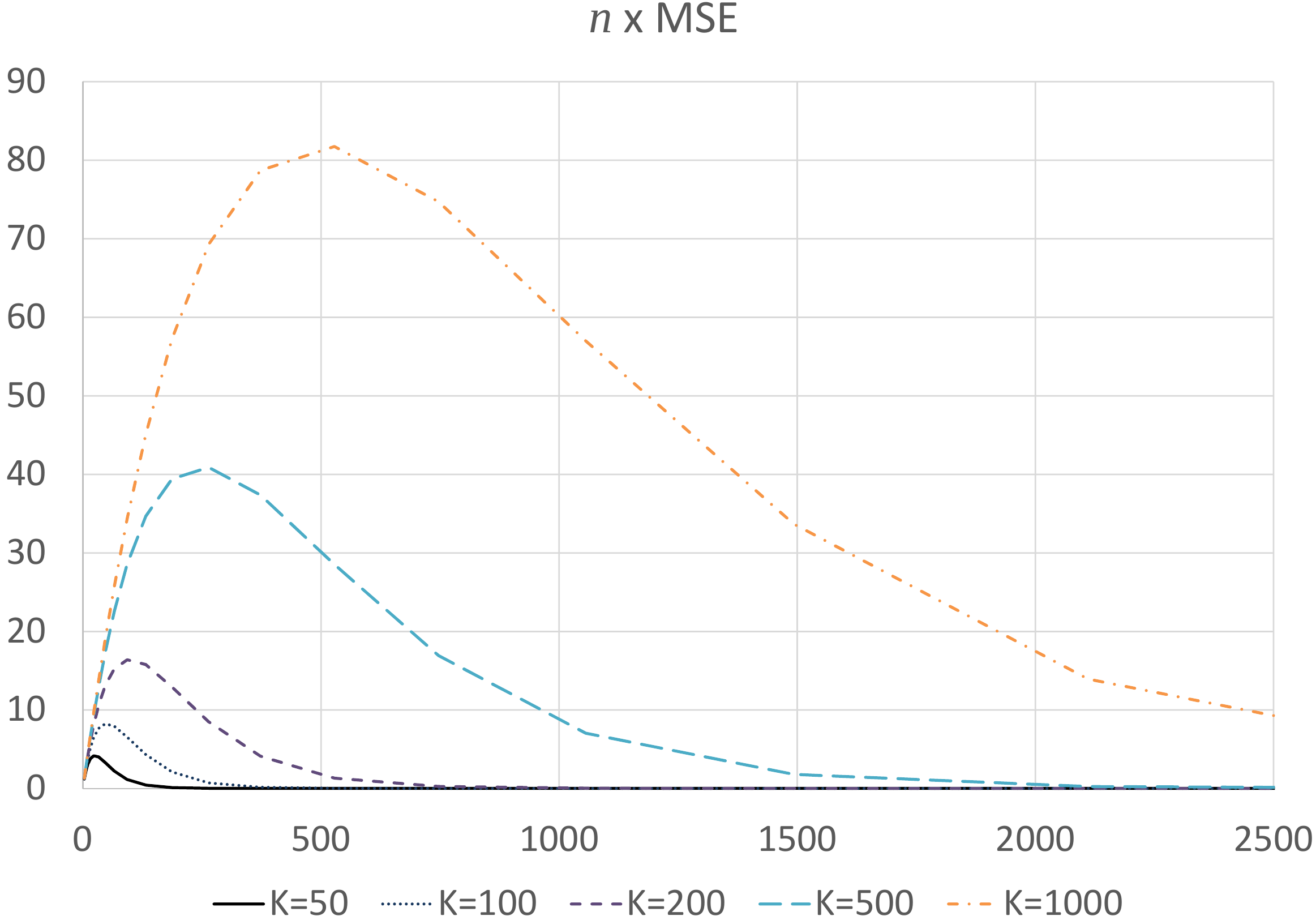}
\end{tabular}
\end{center}
\caption{nMSE of estimators against sample size.} \label{fig:exp}
\end{figure}

This subsection corroborates our analysis with simulation results that empirically demonstrate the impact of key quantities on the MSE of the two estimators.  Two sets of experiments are done, corresponding to the left and right panels in Figure~\ref{fig:exp}.  In all experiments, we repeat the data-generation process (with $\pi_D$) 10,000 times, and compute the MSE of each estimator.  All reward distributions are normal distributions with $\sigma^2=0.01$ and different means.  We then plot normalized MSE (MSE multiplied by sample size $n$), or nMSE, against $n$.

The first experiment is to compare the finite-time as well as asymptotic accuracy of $\vlr$ and $\vreg$. We choose $K=10$, $r_\Phi(a)=a/K$, $\pi(a)\propto a$.  Three choices of $\pi_D$ are used: (a) $\pi_D(a)\propto a$, (b) $\pi_D(a)=1/K$, and (c) $\pi_D(a)\propto (K-a)$.  These choices lead to increasing values of $V_2$ (with $V_1$ approximately fixed).  Clearly, the nMSE of $\vlr$ remains constant, equal to $V_1+V_2$, as predicted in Proposition~\ref{prop:lrmse}.  In contrast, the nMSE of $\vreg$ is large when $n$ is small, because of the high bias, and then quickly converges to the asymptotic minimax rate $V_1$ (Theorem~\ref{thm:regrminimax}, part~iii).  As $V_2$ can be arbitrarily larger than $V_1$, it follows that $\vreg$ is preferred over $\vlr$, as least for sufficiently large $n$ that is needed to drive the bias down.  It should be noted that in practice, after $D^n$ is generated, it is easy to quantify the bias of $\vreg$ simply by identifying the set of actions $a$ with $n(a)=0$.

The second experiment is to show how $K$ affects the nMSE of $\vreg$.  Here, we choose $\pi_D=1/K$, $r_\Phi(a)=a/K$, $\pi(a)\propto a$, and vary $K\in\{50,100,200,500,1000\}$.  As Figure~\ref{fig:exp}~(right) shows, a larger $K$ gives $\vreg$ a harder time, which is consistent with Theorem~\ref{thm:regrminimax} (part~\ref{thm:regrminimax:part1}).  Not only does the maximum nMSE grow approximately linearly with $K$, the number of samples needed for nMSE to start decreasing also scales roughly as $(K-1)/2$, as indicated by part~\ref{thm:regrminimax:part2} of Theorem~\ref{thm:regrminimax}.  \todol{Why $(K-1)/2$ instead of $K/2$?}
%\todol{Maybe Lemma~\ref{lem:emptybincount} can be mentioned here?}

%!TEX root =  paper.tex
\section{Extensions} \label{sec:extension}

In this section, we consider extensions of our previous results to contextual bandits and Markovian Decision Processes, while implications to semi-supervised learning 
\citep{ZhuGo09} 
are discussed in the supplementary material.

\subsection{Contextual Bandits}

The problem setup is as follows: In addition to the finite action set $\A = \{1,2,\ldots,K\}$, we are also given a context set $\X = \{1,2,\ldots,M\}$.
A policy now is a map $\pi: \X \to [0,1]^\A$ such that for any $x\in \X$, $\pi(x)$ is a probability distribution over the action space $\A$. For notational convenience, we will use $\pi(a|x)$ instead of $\pi(x)(a)$. 
%Thus, for any $x\in \X,a\in \A$, $\pi(a|x)\in [0,1]$, and $\sum_a \pi(a|x) = 1$.
The set of policies over $\X$ and $\A$ will be denoted by $\Pi(\X,\A)$.
The process generating the data $D^n=\{(X_i,A_i,R_i)\}_{1\le i\le n}$ is described by the following: $(X_i,A_i,R_i)$ are independent copies of $(X,A,R)$, where $X\sim \mu(\cdot)$, $A\sim \pi_D(\cdot|X)$ and $R \sim \Phi(\cdot|A,X)$ for some unknown family of distributions $\{\Phi(\cdot|a,x)\}_{a\in \A,x\in \X}$ and \emph{known} policy $\pi_D \in \Pi(\X,\A)$ and context distribution $\mu$.  For simplicity, we fix $\rmax=1$.

We are also given a \emph{known} target policy $\pi\in \Pi(\X,\A)$ and want to estimate its value, 
$\vtruectx\defeq \E_{X\sim \mu,A\sim\pi(\cdot|X),R\sim\Phi(\cdot|A,X)}[R]$ based on the knowledge of 
$D^n$, $\pi_D$, $\mu$ and $\pi$, where the quality of an estimate $\hat{v}$ constructed based on $D^n$ (and $\pi,\pi_D,\mu$) is measured by its mean squared error, $\MSE{\hat{v}}\defeq \EE{ (\hat{v}-\vtruectx)^2 }$, just like in the case of contextless bandits. \todoc{mean squared error, or mean-squared error? in any ways, write it one way.}
Let $\sigma^2_\Phi(x,a)  = \V(R)$ for $R\sim \Phi(\cdot|x,a)$, $x\in \X,a\in \A$.
An estimator $\Aalg$ can be considered as a function that maps $(\mu,\pi,\pi_D,D^n)$ to an estimate of $\vtruectx$, denoted $\valg(\mu,\pi,\pi_D,D^n)$.  
Fix $\sigma^2 \defeq (\sigma^2(x,a))_{x\in \X, a\in \A}$.
The minimax optimal risk subject to $\sigma^2_\Phi(x,a)\le \sigma^2(x,a)$ for all $x\in\X,a\in \A$ is defined by
$
R_n^* (\mu,\pi,\pi_D,\sigma^2):= \inf_\Aalg \sup_{\Phi:\sigma^2_\Phi \le \sigma^2} \E \left[(\valg(\mu,\pi,\pi_D,D^n)-\vtruectx)^2\right]\,.
$
 
The main observation is that the estimation problem for the contextual case can actually be reduced to the contextless bandit case by treating the context-action pairs as ``actions'' belonging to the product space $\X \times \A$.
For any policy $\pi$, by slightly abusing notation, let $(\mu\otimes \pi)(x,a) = \mu(x) \pi(a|x)$ be the joint distribution of $(X,A)$ when $X\sim \mu(\cdot)$, $A\sim \pi(\cdot|X)$.
This way, we can map any contextual policy evaluation problem defined by $\mu$,$\pi_D$, $\pi$, $\Phi$ and a sample size $n$ into a contextless policy evaluation problem defined by $\mu \otimes \pi_D$, $\mu \otimes \pi$, $\Phi$ with action set  $\X\times \A$.  Therefore, with $V_1$ and $V_2$ defined similarly, one can conclude the following results:
%
%
%
%Let $X\sim \mu(\cdot)$, $A\sim \pi_D(\cdot|X)$, $R \sim \Phi(\cdot|X,A)$ and define
%\begin{eqnarray*}
%%\pi_D^* &\defeq& \min_a \pi_D(a)\,, \\
%V_1 &\defeq& \E\left[\V\left(\frac{\pi(A|X)}{\pi_D(A|X)}R|X,A\right)\right] = \sum_{x,a} \mu(x) \frac{\pi^2(a|x)}{\pi_D(a|x)}\sigma_\Phi^2(x,a)\,, \\
%V_2 &\defeq& \V\left(\E\left[\frac{\pi(A|X)}{\pi_D(A|X)}R|X,A\right]\right) 
%= \V\left( \frac{\pi(A|X)}{\pi_D(A|X)} r_\Phi(X,A) \right)\,.
%%= \sum_{x,a}\frac{\pi^2(a|x)}{\pi_D(a|x)}r_\Phi(x,a)^2 - (\vtruectx)^2 \,.
%\end{eqnarray*}
%Note that $V_1$ and $V_2$ are a function of $\mu,\pi_D$ and $\pi$.
%In this case the LR and REG estimators take the following form
%\[
%\vlr = \frac1n \sum_{i=1}^n \frac{\pi(A_i|X_i)}{\pi_D(A_i|X_i)} R_i
%%\]
%%while the regression estimator takes the form 
%%\[
%\quad{\text{ and }}
%\vreg = \sum_{x,a}\mu(x) \pi(a|x) \hat{r}(x,a)\,,
%\]
%where now 
%%$\hat{r}(x,a) = 0$ if $(x,a)$ did not occur in $D^n = \{ (X_i,A_i,R_i) \}_{i=1,\ldots,n}$, and $\hat{r}(x,a) = \sum_{i} \one{X_i=x,A_i=a}R_i/ \sum_{i} \one{X_i=x,A_i=a}$, otherwise.
%$\hat{r}(x,a) = \sum_{i} \one{X_i=x,A_i=a}R_i/ \sum_{i} \one{X_i=x,A_i=a}$.
%Note that  the regression estimator can also be computed in $O(n)$ time independently of the size of $\X$ and $\A$, based on rewriting it as a likelihood ratio estimator when $\pi_D$ is replaced by its empirical estimates (cf.~\eqref{eq:reglr}).
%
%The mapping from contextual to contextless bandits gives rise to the following result, combined with \cref{thm:lb}, \cref{prop:lrmse} and \cref{thm:regrminimax}:
\begin{theorem}
Pick any $n>0$, $\mu$, $\pi_D$, $\pi$ and $\sigma^2$. Then, one has
$
R_n^* (\mu,\pi,\pi_D,\sigma^2)= \Omega\left( \max_{B \subset\X\times\A}\{\sum_{(x,a)\in B}\mu(x)\pi(a|x)\}^2\{1-\sum_{(x,a)\in B}\mu(x)\pi_d(a|x)\}^n + V_1/n\right),
$
%Further, the MSE of the likelihood ratio estimator is 
$\MSE{\vlr} = (V_1+V_2)/n$, 
%while the MSE of the regression estimator satisfies 
and $\MSE{\vreg} \le C R_n^* (\mu,\pi,\pi_D,\sigma^2)$, for $C=MK\{\min(4MK,\max_{x,a}r_\Phi^2(a)/\sigma_\Phi^2(a))+5\}R_n^*(\mu,\pi,\pi_D,\sigma^2)$.
Furthermore, the MSE of the regression estimator approaches the minimax risk as sample size grows to infinity.
\end{theorem}

\subsection{Markov Decision Processes}

Similarly, results in Section~\ref{sec:mab} can be naturally extended to fixed-horizon, finite Markov decision processes (MDPs).  Here, an MDP is described by a tuple $M=\langle \X,\A,P,\Phi,\nu,H \rangle$, where $\X=\{1,\ldots,N\}$ is the set of states, $\A=\{1,\ldots,K\}$ the set of actions, $P$ the transition kernel, $\Phi:\X\times\A\mapsto\R$ the reward function, $\nu$ the start-state distribution, and $H$ the horizon.
%Let $K=|\A|$ be the number of actions and $N=|\X|$ the number of states.  
A policy $\pi:\X\mapsto[0,1]^K$ maps states to distributions over actions, and we use $\pi(a|x)$ to denote the probability of choosing action $a$ in state $x$.  
%The set of policies over $\X$ and $\A$ is denoted by $\Pi(\X,\A)$.  
Given a policy $\pi\in\Pi(\X,\A)$, a trajectory of length $H$, denoted $T=(X,A,R)$ (for $X\in\X^H$, $A\in\A^H$, and $R\in\R^H$), is generated as follows: $X(1)\in\nu(\cdot)$; for $h\in\{1,\ldots,H\}$, $A(h)\sim\pi(\cdot|X(h))$, $R(h)\sim\Phi(\cdot|X_{(h)},A_{(h)})$, and $X(h+1)\sim P(\cdot|X_{(h)},A_{(h)})$.  The policy value is defined by 
$v^\pi_\Phi \defeq \E_T[\sum_{h=1}^H R(h)]$.  For simplicity, we again assume $\rmax=1$.
The off-policy evaluation problem is to estimate $v^\pi_\Phi$ from data $D^n=\{T_t\}_{1 \le t \le n}$, where each trajectory $T_t$ is independently generated by an exploration policy $\pi_D\in\Pi(\X,\A)$.  Here, we assume the reward distribution $\Phi$ is unknown; other quantities including $\nu$, $P$, $H$, $\pi$, and $\pi_D$ are all known.  Again, we measure the quality of an estimate $\hat{v}$ by its mean squared error: $\MSE{\hat{v}}\defeq\left[(\hat{v}-v^\pi_\Phi)^2\right]$.
By considering a length-$H$ trajectory of state-actions as an ``action,'', one can apply the results as in the previous subsection to conclude the following:
\begin{theorem}
Pick any $n>0$, $\nu$, $\pi_D$, $\pi$, $P$, $H$, and $\sigma^2$.  Then, one has $R_n^*(\nu,\pi,\pi_D,P,H,\sigma^2)=\Omega\left( \max_{B\subset\T}\{\sum_{(x,a)\in\T}\mu(x,a)\}^2 \{1-\sum_{(x,a)\in\T}\mu_D(\tau)\}^n + V_1/n\right)$, $\MSE{\vlr}=(V_1+V_2)/n$, and $\MSE{\vreg} \le CR_n^*(\nu,\pi,\pi_D,P,H,\sigma^2)$ for $C = N^{H+1}K^H\{\min(4N^{H+1}K^H,\max_{(x,a)\in\T}\frac{r_\Phi^2(x,a)}{\sigma_\Phi^2(x,a)}) + 5\}$.  Moreover, there are cases where such an exponential dependence is unavoidable.  Finally, the MSE of the regression estimator approaches the minimax risk as sample size $n$ grows to infinity.
\end{theorem}

%!TEX root =  paper.tex
\section{Conclusions}
\label{sec:conclude}

We have studied the fundamental problem of finite off-policy evaluation.
Despite its importance, it appears that ours are the first results for the finite-sample setting.
While the simplest estimator which uses importance weights (called LR) was found to be sensitive to the magnitude of importance weights, the regression estimator (REG), which estimates the mean rewards for each actions, was found to be less exposed to this value. While the sensitivity of LR is a ``folk theorem'', we have not seen this result formally proven in the literature. We also found that the REG estimator has different qualities: It is minimax optimal up to a constant, which is the minimum of the squared number of actions, $K^2$, and the maximal inverse reward variance. We showed that the dependence on the number of actions cannot in general be removed. There is still a gap of factor of $K$ between our lower and upper bounds. We conjecture that the lower bound shows the correct order (which seems to be confirmed by the experiments). 
While it is not hard to design estimators that combine LR and REG, we did not find these attractive as they cannot be shown to be near-optimal in the above sense. 
Hence, it remains open to design an estimator which is minimax optimal up to a universal constant factor.
One starting point is to investigate the many alternate estimators proposed in the literature (e.g., LR with clipped weights, or dividing by the sum of weights instead of dividing by $n$).
While in the paper we focused on the simplest contextless, finite setting, we showed that our results have implications to other, more contextual settings. However, we have only scratched the surface here: Much more work is needed, however, to provide a fuller analysis of sample based off-policy evaluation in these settings.

\bibliography{refs}

\newpage

\appendix
\section{Technical Details} \label{sec:app}

The appendix collects miscellaneous results that are needed in the main body of the text.
%!TEX root =  paper.tex
\subsection{Proof of the Second Part of \cref{thm:lb}}
\label{sec:cr}
\todoc{Must check for consistent capitalization of section titles}
We provide here a full proof of the second part of \cref{thm:lb}.
First, we need some background.
Let $\X=(\X,\mathcal{A})$ be a measurable space, $\Theta\subset \R^K$ open,
$p \equiv p(\cdot;\theta)_{\theta\in\Theta}$ be a family of densities with respect to $\nu$, 
a $\sigma$-finite measure on $\X$
such that $p(\cdot;\theta)$ is defined on the closure $\bar\Theta$ of $\Theta$ and $p$ is measurable on the product $\sigma$-algebra of $\X\times \Theta$ where $\Theta$ is equipped with the $\sigma$-algebra of Borel sets.
Denote by $F(\theta) =
%	 \int ( \frac{\partial p}{\partial \theta}(x;\theta) )
%											(\frac{\partial p}{\partial \theta}(x;\theta))^\top \frac{\nu(dx)}{p(x;\theta)}$
	 \int ( \frac{\partial \log p}{\partial \theta}(x;\theta) )
											(\frac{\partial \log p}{\partial \theta}(x;\theta))^\top p(x;\theta) \nu(dx)$
be the Fisher information matrix of $p$ at $\theta$.
The family $p$ is called \emph{regular} if the following hold:
\begin{enumerate}[(a)]
\item $p(x; \theta)$ is a continuous function on $\Theta$ for $\nu$-almost all $x$;
\item $p$ possesses finite Fisher's information at each point $\theta\in \Theta$;
\item the function $\psi(\cdot;\theta)$ is continuous in the space $L^2(\nu)$.
\end{enumerate}
\if0
%Denote by $\E_\theta$ the expectation with respect to measure $p(\cdot;\theta)\nu(\cdot)$.
\begin{theorem}[Cramer-Rao Lower Bound]
\label{thm:cr}
Let $p = (p(x;\theta))_{x\in \X,\theta\in \Theta}$ 
be a regular family of densities with information matrix 
$F(\theta)\succeq 0$, $\theta\in \Theta$. \todoc{Or only at $\theta$?}
Pick $\theta\in \Theta$ and 
let $t:\X \to \R^K$ be measurable such that
$u\mapsto \int (t(x)-u)^2 p(x;u)\nu(dx)$ is bounded in a neighborhood of $\theta$. 
Then, the bias $d(u) = \int t(x) p(x;u) \nu(dx) - u$ 
is continuously differentiable in a neighborhood of the point $\theta\in\Theta$ and
\[
\EE{ (t(X) - \theta)(t(X)-\theta)^\top } \succeq \left(I_{K\times K}+\frac{\partial d(\theta)}{\partial \theta}\right) F^{-1}(\theta) 
\left(I_{K\times K}+\frac{\partial d(\theta)}{\partial\theta} \right)+ d(\theta) d(\theta)^\top\,,
\]
where $I_{K\times K}$ is the $K\times K$ identity matrix,
$X\sim p(\cdot;\theta)\nu(\cdot)$.
\end{theorem}
From this result, it is not hard to derive the following:
\fi
\begin{theorem}[Cramer-Rao Lower Bound]
\label{thm:cr2}
Let $p = (p(x;\theta))_{x\in \X,\theta\in \Theta}$ 
be a regular family of densities with information matrix 
$F(\theta)\succeq 0$, $\theta\in \Theta$. \todoc{Or only at $\theta$?}
%Let $p$ be as in \cref{thm:cr}, 
Pick $\theta\in \Theta$ and assume that 
$\psi:\Theta \to \R$,
$t:\X \to \R$ are measurable such that
$u\mapsto \int (t(x)-\psi(u))^2 p(x;u)\nu(dx)$ is bounded in a neighborhood of $\theta$
and $\psi$ is differentiable.
Then, the bias $d(u) = \int t(x) p(x;u) \nu(dx) - \psi(u)$ 
is continuously differentiable in a neighborhood of the point $\theta\in\Theta$ and
\begin{align}
\label{eq:cr}
\EE{ (t(X) -\psi( \theta))^2 } 
& 
\ge \left(\psi'(\theta)+d'(\theta)\right)^\top F^{-1}(\theta) 
\left(\psi'(\theta)+d'(\theta)\right)+ \norm{d'(\theta)}_2^2\,,
\end{align}
where $X\sim p(\cdot;\theta)\nu(\cdot)$.
\end{theorem}
The proof follows closely that of Theorem~7.3 of \citet{Ibramigov81StatEstBook}, which states this result for $\psi(\theta)=\theta$ (and thus $k=K$) only, and is hence omitted.
\newcommand{\hv}{\hat{v}}

With this, we can present the details of the proof of the second part of \cref{thm:lb}.
Choose $\X = \A \times \R$, $p(a,y;\theta) = \pi_D(a) \phi( y; r(a), \sigma^2(a))$, where 
$\theta = (r(a))_{a\in A}$ is the unknown parameter to be estimated, and $\phi(\cdot;\mu,\sigma^2)$ is the density of the normal distribution with mean $\mu$ and variance $\sigma^2$, $\Theta = \R$.
It is easy to see that $p = (p(\cdot;\theta)_{\theta\in \Theta})$ is a regular family.
Let the quantity to be estimated be $\psi(\theta) = \sum_a \pi(a) r(a)$.
By~\cref{thm:cr2}, for any estimator $\Aalg$, if
 $\hv_n$ is the estimate constructed by $\Aalg$ 
 based on the data $D^n$ generated from $p(\cdot;\theta)$ in an i.i.d. fashion,
the bias $d_n(\theta) = \E_{\theta}[\hv_n]$ is differentiable on $\Theta$ and
\begin{align}
\label{eq:mselb}
\MSE{\hat{v}} 
& \ge \frac{1}{n} 
 \left(\psi'(\theta)+d_n'(\theta)\right)^\top F^{-1}(\theta) 
\left(\psi'(\theta)+d_n'(\theta)\right)+ \norm{d_n'(\theta)}_2^2\,,
\end{align}
where $F(\theta)$ is the Fisher information matrix underlying $p(\cdot;\theta)$.
If $\MSE{\hat{v}_n} \not\to 0$ then $\limsup_{n\to\infty} \frac{\MSE{\hv_n}}{V_1/n} =+\infty$.
Hence, it suffices to consider $\Aalg$ such that $\MSE{\hat{v}_n} \to 0$.
Then, by~\eqref{eq:mselb}, $0\le \norm{d_n'(\theta)}_2^2 \le \MSE{\hat{v}_n}$, hence we also have $\norm{d_n'(\theta)}_2^2\to 0$.

Now, a direct calculation shows that 
$F(\theta) = \diag( \ldots, \pi_D(a)/\sigma^2(a), \ldots)$ and $\psi'(\theta) = \pi$. 
Hence, $\psi'(\theta)^\top F^{-1}(\theta) \psi'(\theta) = V_1$
and using again \eqref{eq:mselb},
\[
\limsup_{n\to\infty} \frac{\MSE{\hv_n}}{V_1/n} 
\ge 1 - 2 \limsup_{n\to \infty} \frac{(d_n'(\theta))^\top F^{-1}(\theta) \psi'(\theta)}{V_1}  = 1\,,
\]
finishing the proof. \todoc{This last step uses that $\theta$ is finite dimensional.}

%\newpage
%!TEX root =  paper.tex
\subsection{Proof for Proposition~\ref{prop:lrmse}}

In the proof, we use the shorthand $\vlr$ for $\vlr(\pi,\pi_D,D^n)$.  As already noted,
the estimator is unbiased, so its MSE equals its variance.  Since samples in $D^n$ are independent, we have
\[
\V(\vlr)=\frac{1}{n} \V\Big(\frac{\pi(A)}{\pi_D(A)}R\Big).
\]
The law of total variance implies
\beqan
 \V(\vlr) &=& \frac{1}{n}\E\Big[\V\Big(\frac{\pi(A)}{\pi_D(A)}R | A\Big)\Big] + \frac{1}{n} \V\Big[\E\Big(\frac{\pi(A)}{\pi_D(A)}R | A \Big)\Big].
\eeqan

The first term equals
\[
\frac{1}{n}\E\Big[\Big(\frac{\pi(A)}{\pi_D(A)}\Big)^2 \sigma^2(A) | A\Big)\Big] = \frac{1}{n}\sum_a \pi_D(a)\frac{\pi^2(a)}{\pi_D^2(a)} \sigma^2(a) = \frac{V_1}{n} .
\]

The second term is
\begin{eqnarray*}
\frac{1}{n}\V\left[\frac{\pi(A)}{\pi_D(A)}r_\Phi(A)\right]
%&=& \frac{1}{n}\left[ \sum_a \pi_D(a)\frac{\pi(a)^2}{\pi_D(a)^2}r_\Phi(a)^2 - \left(\sum_a\pi_D(a)\frac{\pi(a)}{\pi_D(a)}r_\Phi(a)\right)^2\right] \\
&=& \frac{1}{n}\left[ \sum_a \frac{\pi^2(a)}{\pi_D(a)}r_\Phi^2(a) - \left(v_\Phi^\pi\right)^2\right] = \frac{V_2}{n}.
\end{eqnarray*}
Combining the two above completes the proof.

\subsection{Proof for Proposition~\ref{prop:regmse}}

We note that the MSE is equal to the sum of the variance and the squared bias.
Let us abbreviate $\vreg(\pi,D^n)$ by $\vreg$.
First, notice that this estimate is (slightly) biased:
\beqan
\E[\vreg] &=& \sum_a \pi(a) \E[\hat{r}(a)] \\
&=& \sum_a \pi(a) \E[\E[\hat{r}(a)|n(a)]]\\
&=& \sum_a \pi(a) \E[r_\Phi(a) \1\{n(a)>0\} + 0\times \1\{n(a)=0\}]\\
&=& \sum_a \pi(a) r_\Phi(a) (1-p_{a,n}).
\eeqan
Thus, the squared bias can be bounded as follows:
%\begin{eqnarray*}
\[
\left(\E[\vreg]-v_\Phi^\pi\right)^2 = \left(\sum_a\pi(a)r_\Phi(a)p_{a,n}\right)^2\,.
%\le \left(\sum_a\pi(a)|r_\Phi(a)|\right)^2 \max_ap_{a,n}^2
%=(v_{|\Phi|}^\pi)^2 \max_a p_{a,n}^2 \,.
\]
%\end{eqnarray*}

For the variance term, we again use the law of total variance to yield:
\[
\V(\vreg) = \E[\V(\vreg|n(1),\ldots,n(K))] + \V(\E[\vreg|n(1),\ldots,n(K)]).
\]
 
Now, conditioned on $n(1),\ldots,n(K)$, the estimates $\{\hat{r}(a)\}_{a\in\A}$ are independent, so, by distinguishing the case $n(a)>0$ (for which the variance of $\hat{r}(a)$ is $\sigma^2(a)/n(a)$) from the other case $n(a)=0$ (for which this variance is $0$), we have
\[
\V(\vreg|n(1),\ldots,n(K)) = \sum_a \pi^2(a) \Big( \frac{\sigma^2(a)}{n(a)}\1\{n(a)>0\} + 0 \times \1\{n(a)=0\}\Big).
\] 
Thus,
$$\E[\V(\vreg|n(1),\ldots,n(K))] = \sum_a \pi^2(a) \sigma^2(a) \E\Big[\frac{1}{n(a)}\1\{n(a)>0\}\Big].$$

%Notice that $$\E\Big[\frac{1}{n_i}\1\{n_i>0\}\Big] = \sum_{t>0}\frac{1}{t} \P(n_i=t) = \sum_{t>0}\frac{1}{t} \Big(\begin{array}{c} n\\t \end{array}\Big) \pi_D(i)^t (1-\pi_D(i))^{n-t}$$

For the second variance term, we also distinguish the case $n(a)>0$, for which $\E[\hat{r}(a)|n(a)]=r_\Phi(a)$, from the case $n(a)=0$, for which $\E[\hat{r}(a)|n(a)=0]$, thus
\[
\E[\vreg|n(1),\ldots,n(K)] = \sum_a \pi(a) (r_\Phi(a) \1\{n(a)>0\} + 0\times \1\{n(a)=0\}),
\]
Hence, $
\V(\E[\vreg|n(1),\ldots,n(K)]) = \V(  \sum_a \pi(a) r_\Phi(a) \1\{n(a)>0\} )$, which by Lemma~\ref{lem:reg-var-ub} implies
\[
\V(  \sum_a \pi(a) r_\Phi(a) \1\{n(a)>0\} ) \le   \sum_a \pi^2(a) r_\Phi^2(a) \,  p_{a,n}(1-p_{a,n})\,.
\]
The proof of the upper bound is then completed by adding squared bias to variance, and using definitions of $V_{0,n}$, $V_1$, and $V_3$.

For the lower bound, use \cref{thm:cr2}.
As mentioned in \cref{sec:cr}, the Fisher information matrix is
$F(\theta) = \diag( \ldots, \pi_D(a)/\sigma^2(a), \ldots)$ and
if the target is $\psi(\theta) = \sum_a \pi(a) r(a)$,  $\psi'(\theta) = \pi$. Calculating the derivative of the bias and plugging into~\eqref{eq:cr}, we get the result.
%\begin{align*}
%\MSE{ \vreg }
%& =  \left(\E[\vreg]-v_\Phi^\pi\right)^2 +  \V(\vreg) \\
%& =
% \left(\sum_a\pi(a)r_\Phi(a)p_{a,n}\right)^2 + 
% \V(  \sum_a \pi(a) r_\Phi(a) \1\{n(a)>0\} ) \\
% & \qquad +
% \sum_a \pi^2(a) \sigma^2(a) \E\Big[\frac{1}{n(a)}\1\{n(a)>0\}\Big]\,.
%\end{align*}
%%By  $ \V(  \sum_a \pi(a) r_\Phi(a) \1\{n(a)>0\} ) \le   \sum_a \pi^2(a) r_\Phi^2(a) \,  p_{a,n}(1-p_{a,n})$. 
%This, together with 
%the definitions of $V_{0,n}$, $V_1$ and $V_{3,n}$ gives
%\[
%\MSE{ \vreg } \le  V_{0,n} + \frac{V_1+V_{3,n}}{n}\,,
%\]
%which is the final result.
%This quantity equals the constant $v_\Phi^\pi=\sum_a \pi(a) r_\Phi(a)$ when $n(a)>0$ for all $a$, and is smaller (but still positive) if $n(a)=0$ for some $a$. Thus,
%%(using Lihong's idea for bounding the variance)
%its variance is bounded by $(\sum_a \pi(a) |r_\Phi(a)|)^2$ times the variance of a Bernoulli r.v.~with parameter $\P(\min_a n(a)>0)$: \todo{Add details for this step}
%\[
%\V(\E[\vreg|n(1),\ldots,n(K)])\leq \Big(\sum_a \pi(a)|r_\Phi(r)|\Big)^2 \P(\min_a n(a)>0) \P(\min_a n(a)=0) .
%\]
%Adding the squared bias, simplifying and reordering the terms we deduce the result.

\subsection{Proof for Lemma~\ref{lem:invmoment}}

For convenience, the lemma is restated here.

\textbf{Lemma~\ref{lem:invmoment}.} \textit{
Let $X_1,\ldots,X_n$ be $n$ independent Bernoulli random variables with parameter $p>0$. 
Letting $S_n = \sum_{i=1}^n X_i$, $\hat{p} = S_n/n$, $Z = \frac{\one{S_n>0}}{\hat{p}} - \frac{1}{p}$, we have for any $n$ and $p$ that
\begin{align}
\label{eq:invmoment-1}
\EE{ Z } \le \frac{4}{p}\,.
\end{align}
Further, when $np\ge 34$,
\begin{align}
\label{eq:invmoment-2}
\EE{ Z } \le \frac{2}{p} \sqrt{\frac{2}{np}} \left( \sqrt{\frac32 \ln\left(\frac{np}{2}\right)}+1 \right)\,.
\end{align}
}

\begin{proof}
%Let $Z = \frac{\one{S_n>0}}{\hat{p}} - \frac{1}{p}$.
According to the multiplicative Chernoff bound for the low tail (cf. \cref{lem:chernoff-lower} in the Appendix), 
for any $0<\delta\le 1$, 
with probability at least $1-\delta$, we have
\[
\hat{p} \ge p - \sqrt{\frac{2p}{n}\ln\frac{1}{\delta}} .
\]
Denote by  $\mathcal{E}_\delta$  the event when this inequality holds.
Assuming 
\begin{align}\label{eq:deltalimit}
\frac{2}{np}\ln\frac{1}{\delta}\le 1/4\,,
\end{align}
thanks to $1/(1-x) \le 1+2x$ which holds for any $x\in [0,1/2]$,
on $\mathcal{E}_\delta$ 
we have
\begin{align*}
Z \le \frac{1}{\hat{p}} - \frac{1}{p}
 \le \frac{1}{p}\left( \frac{1}{1-\sqrt{\frac{2}{np} \ln\frac{1}{\delta}}}  - 1 \right) 
 \le \frac{2}{p} \sqrt{\frac{2}{np} \ln\frac{1}{\delta}} \,.
\end{align*}
Then, since $Z \le n$, we have for every $\delta$ satisfying~\eqref{eq:deltalimit} that
\begin{align}
\E[Z] \le \frac{2}{p}\sqrt{\frac{2}{np}\ln\frac{1}{\delta}} + \delta n 
= \frac{2}{p}\left(\sqrt{\frac{2}{np}\ln\frac{1}{\delta}} + \frac{np}{2} \delta \right)
= \frac{2}{p} f\left( \frac{np}{2},\delta\right)\,,
\label{eq:zbound}
\end{align}
where $f(u,\delta) = \sqrt{\frac{1}{u} \ln \frac{1}{\delta}} + u \delta$.
Hence, it remains to choose $\delta$ 
to approximately minimize $f(u,\delta)$ subject to the constraint $\delta\ge e^{-u/4}$ (due to \eqref{eq:deltalimit}). 
First, note that if we choose $\delta = e^{-u/4}$, then $f(u,e^{-u/4}) \le \frac12 + u e^{-u/4} < 2$, showing that $\E{Z} \le 4/p$, proving the first part of the result.

To get the second part, we choose 
%$\delta$ in a way that depends on $u=np/2$: % a $u$-dependent manner.
%In particular, assuming $u \ge 17$, let 
$\delta = u^{-3/2}$, which satisfies \eqref{eq:deltalimit} since $u^{-3/2}\ge e^{-u/4}$ for $u\ge17$.
Then, $f(u,u^{-3/2}) = u^{-1/2} \left( \sqrt{\frac32 \ln(u)}+1 \right)$.
Plugging this into~\eqref{eq:zbound} finishes the proof.
%gives 
%\[
%\EE{Z} \le \frac{2}{p} \sqrt{\frac{2}{np}} \left( \sqrt{\frac32 \ln\left(\frac{np}{2}\right)}+1 \right)\,,
%\]
%which holds when $np\ge 34$, finishing the proof.
\end{proof}

\subsection{Technical Lemmas}

\begin{lemma}[]
\label{lem:reg-var-ub}
Using notation from Section~\ref{sec:reg}, and  $w_a = \pi(a) r_\Phi(a)$ one has
\[
V^* \defeq %\left(\sum_a   w_a p_{a,n}\right)^2 +
\V\left(\sum_a\pi(a)r_\Phi(a)\1\{n(a)>0\}\right)
\le 
\sum_{a\in \A} w_a^2\,  p_{a,n}  (1-p_{a,n}) 
% \left(\sum_a |w_a| \sqrt{p_{a,n}}\right)^2
\]
provided that $r(a)\ge 0$ for all action $a\in \A$.
\end{lemma}

\begin{proof}
Let $X_a = \one{n(a)>0}$.
First, note that $\EE{X_a} = p_{a,n}$ and so
\begin{align*}
\V\left(\sum_{a\in \A} w_a\1\{n(a)>0\}\right)
& = \EE{\Big\{\sum_{a\in \A} w_a(X_a-p_{a,n}) \Big\}^2 }\\
& = \sum_{a,b\in \A} w_a w_b\, \EE{ (X_a-p_{a,n}) (X_b-p_{b,n}) } \\
& \le \sum_{a\in \A} w_a^2\, \EE{ (X_a-p_{a,n})^2} & (\text{negative association}) \\
& = \sum_{a\in \A} w_a^2\,  p_{a,n}  (1-p_{a,n}) \,.
\end{align*}
\end{proof}

\begin{lemma}[Multiplicative Chernoff Bound for the Lower Tail, Theorem~4.5 of \citet{MiUp05:book}]
\label{lem:chernoff-lower}
Let $X_1,\ldots,X_n$ be independent Bernoulli random variables with parameter $p$, $S_n = \sum_{i=1}^n X_i$.
Then, for any $0\le \beta<1$,
\[
\Prob{\frac{S_n}{n} \le (1-\beta)  p } \le \exp\left( - \frac{\beta^2 n p }{2} \right)\,.
\]
\end{lemma}

%!TEX root =  paper.tex
\section{Extension to Contextual Bandits}
\label{sec:cband}

In this section we consider an extension of our previous results to finite \todoc{Countable, discrete??} contextual bandits.
As we shall soon see, the extension is seamless.
The problem setup is as follows: In addition to the finite action set $\A = \{1,2,\ldots,K\}$, we are also given a context set $\X = \{1,2,\ldots,M\}$.
A policy now is a map $\pi: \X \to [0,1]^\A$ such that for any $x\in \X$, $\pi(x)$ is a probability distribution over the action space $\A$. For notational convenience, we will use $\pi(a|x)$ instead of $\pi(x)(a)$. 
%Thus, for any $x\in \X,a\in \A$, $\pi(a|x)\in [0,1]$, and $\sum_a \pi(a|x) = 1$.
The set of policies over $\X$ and $\A$ will be denoted by $\Pi(\X,\A)$.

The process generating the data $D^n=\{(X_i,A_i,R_i)\}_{1\le i\le n}$ is described by the following: $(X_i,A_i,R_i)$ are independent copies of $(X,A,R)$, where $X\sim \mu(\cdot)$, $A\sim \pi_D(\cdot|X)$ and $R \sim \Phi(\cdot|A,X)$ for some unknown family of distributions $\{\Phi(\cdot|a,x)\}_{a\in \A,x\in \X}$ and \emph{known} policy $\pi_D \in \Pi(\X,\A)$ and context distribution $\mu$.  For simplicity, we fix $\rmax=1$.

We are also given a \emph{known} target policy $\pi\in \Pi(\X,\A)$ and want to estimate its value, 
$\vtruectx\defeq \E_{X\sim \mu,A\sim\pi(\cdot|X),R\sim\Phi(\cdot|A,X)}[R]$ based on the knowledge of 
$D^n$, $\pi_D$, $\mu$ and $\pi$, where the quality of an estimate $\hat{v}$ constructed based on $D^n$ (and $\pi,\pi_D,\mu$) is measured by its mean squared error, $\MSE{\hat{v}}\defeq \EE{ (\hat{v}-\vtruectx)^2 }$, just like in the case of contextless bandits. \todoc{mean squared error, or mean-squared error? in any ways, write it one way.}
  
Let $\sigma^2_\Phi(x,a)  = \V(R)$ for $R\sim \Phi(\cdot|x,a)$, $x\in \X,a\in \A$.
An estimator $\Aalg$ can be considered as a function that maps $(\mu,\pi,\pi_D,D^n)$ to an estimate of $\vtruectx$, denoted $\valg(\mu,\pi,\pi_D,D^n)$.  
Fix $\sigma^2 \defeq (\sigma^2(x,a))_{x\in \X, a\in \A}$.
The minimax optimal risk subject to $\sigma^2_\Phi(x,a)\le \sigma^2(x,a)$ for all $x\in\X,a\in \A$ is defined by
\[
R_n^* (\mu,\pi,\pi_D,\sigma^2):= \inf_\Aalg \sup_{\Phi:\sigma^2_\Phi \le \sigma^2} \E \left[(\valg(\mu,\pi,\pi_D,D^n)-\vtruectx)^2\right]\,.
\]
  
The main observation is that the estimation problem for the contextual case can actually be reduced to the contextless bandit case by treating the context-action pairs as ``actions'' belonging to the product space $\X \times \A$.
For any policy $\pi$, by slightly abusing notation, let $(\mu\otimes \pi)(x,a) = \mu(x) \pi(a|x)$ be the joint distribution of $(X,A)$ when $X\sim \mu(\cdot)$, $A\sim \pi(\cdot|X)$.
This way, we can map any contextual policy evaluation problem defined by $\mu$,$\pi_D$, $\pi$, $\Phi$ and a sample size $n$ into a contextless policy evaluation problem defined by $\mu \otimes \pi_D$, $\mu \otimes \pi$, $\Phi$ with action set  $\X\times \A$.
Let $X\sim \mu(\cdot)$, $A\sim \pi_D(\cdot|X)$, $R \sim \Phi(\cdot|X,A)$ and define
\begin{eqnarray*}
%\pi_D^* &\defeq& \min_a \pi_D(a)\,, \\
V_1 &\defeq& \E\left[\V\left(\frac{\pi(A|X)}{\pi_D(A|X)}R|X,A\right)\right] = \sum_{x,a} \mu(x) \frac{\pi^2(a|x)}{\pi_D(a|x)}\sigma_\Phi^2(x,a)\,, \\
V_2 &\defeq& \V\left(\E\left[\frac{\pi(A|X)}{\pi_D(A|X)}R|X,A\right]\right) 
= \V\left( \frac{\pi(A|X)}{\pi_D(A|X)} r_\Phi(X,A) \right)\,.
%= \sum_{x,a}\frac{\pi^2(a|x)}{\pi_D(a|x)}r_\Phi(x,a)^2 - (\vtruectx)^2 \,.
\end{eqnarray*}
Note that $V_1$ and $V_2$ are a function of $\mu,\pi_D$ and $\pi$.
In this case the LR and REG estimators take the following form
\[
\vlr = \frac1n \sum_{i=1}^n \frac{\pi(A_i|X_i)}{\pi_D(A_i|X_i)} R_i
%\]
%while the regression estimator takes the form 
%\[
\quad{\text{ and }}
\vreg = \sum_{x,a}\mu(x) \pi(a|x) \hat{r}(x,a)\,,
\]
where now 
%$\hat{r}(x,a) = 0$ if $(x,a)$ did not occur in $D^n = \{ (X_i,A_i,R_i) \}_{i=1,\ldots,n}$, and $\hat{r}(x,a) = \sum_{i} \one{X_i=x,A_i=a}R_i/ \sum_{i} \one{X_i=x,A_i=a}$, otherwise.
$\hat{r}(x,a) = \sum_{i} \one{X_i=x,A_i=a}R_i/ \sum_{i} \one{X_i=x,A_i=a}$.
Note that  the regression estimator can also be computed in $O(n)$ time independently of the size of $\X$ and $\A$, based on rewriting it as a likelihood ratio estimator when $\pi_D$ is replaced by its empirical estimates (cf.~\eqref{eq:reglr}).

The mapping from contextual to contextless bandits gives rise to the following result, combined with \cref{thm:lb}, \cref{prop:lrmse} and \cref{thm:regrminimax}:
\begin{theorem}
Pick any $n>0$, $\mu$, $\pi_D$, $\pi$ and $\sigma^2$. Then, one has
$
R_n^* (\mu,\pi,\pi_D,\sigma^2)= \Omega\left( \max_{B \subset\X\times\A}\{\sum_{(x,a)\in B}\mu(x)\pi(a|x)\}^2\{1-\sum_{(x,a)\in B}\mu(x)\pi_d(a|x)\}^n + V_1/n\right),
$
%Further, the MSE of the likelihood ratio estimator is 
$\MSE{\vlr} = (V_1+V_2)/n$, 
%while the MSE of the regression estimator satisfies 
and $\MSE{\vreg} \le C R_n^* (\mu,\pi,\pi_D,\sigma^2)$, for $C=MK\{\min(4MK,\max_{x,a}r_\Phi^2(a)/\sigma_\Phi^2(a))+5\}R_n^*(\mu,\pi,\pi_D,\sigma^2)$.
Furthermore, the MSE of the regression estimator approaches the minimax risk as sample size grows to infinity.
\end{theorem}
%\begin{proof}
%The result follows from the mapping between contextual and contextless problems, combined with \cref{thm:lb}, \cref{prop:lrmse} and \cref{thm:regrminimax}.
%\end{proof}
%!TEX root =  paper.tex
\section{Extension to Markov Decision Processes}
\label{sec:mdp}

In this section, we consider an extension to fixed-horizon, finite Markov decision processes (MDPs), which will be reduced to the bandit problem studied in Section~\ref{sec:mab}.  Here, an MDP is described by a tuple $M=\langle \X,\A,P,\Phi,\nu,H \rangle$, where $\X=\{1,\ldots,N\}$ is the set of states, $\A=\{1,\ldots,K\}$ the set of actions, $P$ the transition kernel, $\Phi:\X\times\A\mapsto\R$ the reward function, $\nu$ the start-state distribution, and $H$ the horizon.
%Let $K=|\A|$ be the number of actions and $N=|\X|$ the number of states.  
A policy $\pi:\X\mapsto[0,1]^K$ maps states to distributions over actions, and we use $\pi(a|x)$ to denote the probability of choosing action $a$ in state $x$.  The set of policies over $\X$ and $\A$ is denoted by $\Pi(\X,\A)$.  Given a policy $\pi\in\Pi(\X,\A)$, a trajectory of length $H$, denoted $T=(X,A,R)$ (for $X\in\X^H$, $A\in\A^H$, and $R\in\R^H$), is generated as follows: $X(1)\in\nu(\cdot)$; for $h\in\{1,\ldots,H\}$, $A(h)\sim\pi(\cdot|X(h))$, $R(h)\sim\Phi(\cdot|X_{(h)},A_{(h)})$, and $X(h+1)\sim P(\cdot|X_{(h)},A_{(h)})$.  The policy value is defined by 
$v^\pi_\Phi \defeq \E_T[\sum_{h=1}^H R(h)]$.  For simplicity, we again assume $\rmax=1$. \todoc{Do we assume $r(x,a)\in [0,\rmax]$ or $v^\pi\in [0,\rmax]$?}

The off-policy evaluation problem is to estimate $v^\pi_\Phi$ from data $D^n=\{T_t\}_{1 \le t \le n}$, where each trajectory $T_t$ is independently generated by an exploration policy $\pi_D\in\Pi(\X,\A)$.  Here, we assume the reward distribution $\Phi$ is unknown; other quantities including $\nu$, $P$, $H$, $\pi$, and $\pi_D$ are all known.  Again, we measure the quality of an estimate $\hat{v}$ by its mean squared error: $\MSE{\hat{v}}\defeq\left[(\hat{v}-v^\pi_\Phi)^2\right]$.

The key observation is that, similarly to the contextual case, the off-policy evaluation problem in fixed-horizon, finite MDPs can be reduced to the multi-armed bandit case.  Specifically, every possible length-$H$ trajectory is an ``augmented action'' belong to the product space $\T=\X^{H+1}\times\A^H$.  The total number of augmented actions is at most $N^{H+1}K^H$.  The distribution over this augmented action space, induced by $\nu$, $P$ and policy $\pi$, is given by:
$
\mu(x(1),\ldots,x(H+1),a(1),\ldots,a(H))\defeq\nu(x(1))\prod_{h=1}^H\pi(a(h)|x(h))P(x(h+1)|x(h),a(h)) \,.
$
This way, the off-policy evaluation problem is reduced to the bandit case with corresponding induced distributions over augmented actions.

For any $(x,a)\in\T$, let $r_\Phi(x,a)\defeq\E[R]$ and $\sigma^2_\Phi(x,a)\defeq\V(R)$, where $R(h)\sim\Phi(\cdot|x(h),a(h))$.  Define the minimax optimal risk subject to constraints $\sigma^2_\Phi(x,a)\le\sigma^2(x,a)$ for all $(x,a)\in\T$ by:
\[
R^*_n(\nu,\pi,\pi_D,P,H,\sigma^2)\defeq\inf_\Aalg\sup_{\Phi:\sigma^2_\Phi\le\sigma^2} \E\left[ \left(\hat{v}_\Aalg(\nu,\pi,\pi_D,P,H,D^n)-v^\pi_\Phi\right)^2 \right]\,.
\]
Similar to previous sections, one may adjust the definitions of quantities like $V_1$ and $V_2$, and conclude the following result using with \cref{thm:lb}, \cref{prop:lrmse} and \cref{thm:regrminimax}:
\begin{theorem}
Pick any $n>0$, $\nu$, $\pi_D$, $\pi$, $P$, $H$, and $\sigma^2$.  Then, one has $R_n^*(\nu,\pi,\pi_D,P,H,\sigma^2)=\Omega\left( \max_{B\subset\T}\{\sum_{(x,a)\in\T}\mu(x,a)\}^2 \{1-\sum_{(x,a)\in\T}\mu_D(\tau)\}^n + V_1/n\right)$, $\MSE{\vlr}=(V_1+V_2)/n$, and $\MSE{\vreg} \le CR_n^*(\nu,\pi,\pi_D,P,H,\sigma^2)$ for $C = N^{H+1}K^H\{\min(4N^{H+1}K^H,\max_{(x,a)\in\T}\frac{r_\Phi^2(x,a)}{\sigma_\Phi^2(x,a)}) + 5\}$.  Finally, the MSE of the regression estimator approaches the minimax risk as sample size $n$ grows to infinity.
\end{theorem}

Finally, it can be shown that in general an exponential dependence on $H$ is unavoidable.  An example is the ``combination lock'' MDP with $N$ states $\X=\{1,\ldots,N\}$ and $K=2$ actions $\A=\{L,R\}$; the start state is $x_*=1$.  In any state $x$, action $L$ takes the learner back to the initial state $x_*$, while action $R$ takes the learner to state $x+1$.  Assume reward is always zero except in state $N$ where it can be $\{0,\rmax\}$.  It is easy to verify that, if there exists constant $p_*$ such that $p_*\le\pi_D(L|x)$ for all $x$, then it takes exponentially many steps to reach state $N$ from $x_*$ under policy $\pi_D$.  Consequently, it requires at least exponentially many trajectories to evaluate a policy $\pi$ that always takes action $R$.

%!TEX root =  paper.tex
\section{Connection to Semi-supervised Learning}
\label{sec:ssl}

In semi-supervised learning one is given a large unlabeled dataset together with a smaller, 
labeled dataset. The hope is that the large unlabeled dataset will help to decrease the error of an estimator
whose job is to predict some value that depends on the unknown distribution generating the data.
Clearly, the off-policy policy evaluation problem can be connected to semi-supervised learning:
Given the data $\{(A_i,R_i)\}_{i=1,\ldots,n}$ generated from $\pi_D$ and $\Phi$, the goal being to predict $\vtrue$.
A large ``unlabelled'' dataset $\{A_j\}_{j=1,\ldots,m}$ with $m\gg n$ helps one to identify $\pi_D$.
Indeed, an intriguing idea is to use $\pi_D$ in some clever way to help improving the prediction of $\vtrue$. 
The most obvious way is to use it in the likelihood ratio estimator. However, as we have shown, 
the MSE of the likelihood ratio estimator can be much larger than that of the regression estimator, which
\emph{does not use  $\pi_D$ even if it is available}. Further, the MSE of the regression estimator is unimprovable, apart from a constant factor for finite sample sizes, while it also rapidly approaches the \emph{optimal} minimax MSE as the sample size grow. Hence, it seems unlikely that knowing $\pi_D$ can help in this problem.

Note that the regression estimator can also be thought as the solution to a least-squares regression problem and our results thus have implications for using unlabelled data together with least-squares estimators.
Indeed, if $X_i\in \{0,1\}^K$ is chosen to be the $A_i$s unit vector of the standard Euclidean basis, we can write
$\hat{r} = (X^\top X)^{\dagger} X^\top R$, where $\dagger$ denotes pseudo-inverse,
$X\in \R^{n\times K}$ and $R\in \R^{n}$ are defined by $R = (R_1,\ldots,R_n)^\top$, $X^\top = (X_1,\ldots,X_n)$.
Notice that here $G_n = \frac1n X^\top X = \frac1n\sum_{i=1}^n X_i X_i^\top = \diag( \hat{\pi}_D(1), \ldots, \hat{\pi}_D(K) )$.
Thus, $\frac1n X^\top X $ can be seen as an estimate of $G = \EE{X_i X_i^\top } = 
\diag( {\pi}_D(1), \ldots, {\pi}_D(K) )$.
Having access to a large unlabelled set $U_1,\ldots,U_m$ (i.e., $m\gg n$) coming from the same distribution as the $X_i$s, 
it is tempting to replace $\frac1nX^\top X$ with a ``better estimate'', $H_m = \frac1m \sum_{i=1}^m U_i U_i^\top$.
Taking $m$ to the limit, we see that $H_m$ converges to $G$. Now, replacing $G_n$ with $H_m \approx G$ in the least squares estimate, and then taking the weighted sum of the resulting values with weights $\pi(a)$, we get the likelihood ratio estimator. Again, since this was shown to be inferior to the regression estimator, replacing $G_n$ with $H_m$ sound like an idea of dubious status. In fact, preliminary experiments with simple simulated scenarios confirmed that $G_n$ indeed should not be replaced with $H_m$, even when $m$ is very large in least-squares regression estimation.

\if0
In semi-supervised learning one is given a large unlabeled dataset $\mathcal{U}^m = (U_1,\ldots,U_m)$ together with a smaller, labeled dataset $\mathcal{D}^n = ((X_1,Y_1),\ldots, (X_n,Y_n))$, 
where $(X_i,Y_i)$ is i.i.d., \todoc{Have used this abbreviation?}
 just like $(U_i)$, $\mathcal{U}^m$ and $\mathcal{D}^n$ are independent of each other 
and the distribution of $U_i$ and $X_j$ are identical.
An important special case is when $Y_i$ is real valued and the goal is to learn 
the regression function $m(x) = \EE{Y_1|X_1=x}$ where $x$ sweeps though the range of the random variables $X_i$.
The assumption in semi-supervised learning 
is that having access to a large, unlabelled dataset $\mathcal{U}^m$ helps.
However, our previous results show that sometimes an unlabelled dataset will not help no matter how large it is.
To take things to an extreme, assume that $m=\infty$, so that we can know the distribution $\pi_D(\cdot)$ of the input $X_i$s.
Now, assume that $X_i$ takes values in a discrete space $\X = \{1,\ldots,K\}$.
Clearly, in this case, there is nothing to be gained by knowing $\pi_D$,

To see why, consider the case when $m$ is so large that we can assume that we know the common underlying distribution of $U_i$/$X_j$. Assume that $X_i$ takes values in a finite set $\{1,\ldots,K\}$ and assume that the goal is not to estimate $m(x)$, but to estimate

Let us consider again the (contextless) bandit case.
It is not hard to see 
\fi

%\newpage
%\input{scratch}

\end{document}